%% file: main.tex
\documentclass[pdflatex,sn-mathphys]{sn-jnl}
\usepackage{amsmath}
\usepackage{amssymb}
\usepackage{graphicx}
\usepackage{overpic}
\usepackage{caption}
\usepackage{subcaption}
\usepackage{mathrsfs}
\usepackage{algorithm}
\usepackage{cleveref}

\algrenewcommand\alglinenumber[1]{\footnotesize #1.}

\jyear{2022}

\ifpdf
\hypersetup{
  pdftitle={Staying the course},
  pdfauthor={J. M. Bello-Rivas, A. Georgiou, J. Guckenheimer, I. G. Kevrekidis}
}
\fi

\DeclareMathOperator{\Div}{div}
\DeclareMathOperator{\grad}{grad}

\DeclareMathOperator{\Span}{span}
\DeclareMathOperator{\sign}{sign}

\DeclareMathOperator{\adj}{adj}
\DeclareMathOperator{\arcsech}{arcsech}

\renewcommand{\S}{\mathbb{S}}

\renewcommand{\epsilon}{\varepsilon}
\renewcommand{\R}{\mathbb{R}}

\newcommand{\X}{\mathfrak{X}}
\renewcommand{\e}{\mathrm{e}}
\renewcommand{\d}{\mathrm{d}}

\theoremstyle{thmstyleone}
\newtheorem{theorem}{Theorem}[section]
\newtheorem{proposition}[theorem]{Proposition}%

\theoremstyle{thmstyletwo}%
\newtheorem{example}{Example}%
\newtheorem{remark}{Remark}%

\theoremstyle{thmstylethree}%
\newtheorem{definition}{Definition}%

\begin{document}

\title[Staying the course]{Staying the course:
Iteratively locating equilibria of dynamical systems
on Riemannian manifolds defined by point-clouds\thanks{Submitted to the editors \today}}

\author*[1]{\fnm{Juan M.} \sur{Bello-Rivas}}\email{jmbr@jhu.edu}
\author[1]{\fnm{Anastasia} \sur{Georgiou}}\email{ageorgi3@jhu.edu}
\author[2]{\fnm{John} \sur{Guckenheimer}}\email{jmg16@cornell.edu}
\author*[1]{\fnm{Ioannis G.} \sur{Kevrekidis}}\email{yannisk@jhu.edu}

\affil*[1]{\orgdiv{Department of Chemical and Biomolecular Engineering}, \orgname{Whiting School of Engineering, Johns Hopkins University}, \orgaddress{\street{3400 North Charles Street}, \city{Baltimore}, \postcode{21218}, \state{MD}, \country{USA}}}

\affil[2]{\orgdiv{Department of Mathematics}, \orgname{Cornell University}, \orgaddress{
    \city{Ithaca}, \postcode{14853}, \state{New York}, \country{USA}}}

\input{abstract}



\maketitle

\input{introduction}
\input{isoclines}
\input{algorithm}
\input{conclusion}

\backmatter

\input{acknowledgments}

\begin{appendices}

\input{manifold-learning}
\input{closed-form-example}
\input{path-connections}

\end{appendices}

\input{statements-and-declarations}

\bibliography{bibliography}

\end{document}

%% file: abstract.tex
\abstract{
  We introduce a method to successively locate equilibria (steady states) of dynamical systems on Riemannian manifolds.
  The manifolds need not be characterized by an \emph{a priori} known atlas or by the zeros of a smooth map.
  Instead, they can be defined by point-clouds and sampled as needed through an iterative process.
  If the manifold is an Euclidean space, our method follows \emph{isoclines}, curves along which the direction of the vector field $X$ is constant.
  For a generic vector field $X$, isoclines are smooth curves and every equilibrium lies on isoclines.
  We generalize the definition of isoclines to Riemannian manifolds through the use of parallel transport: generalized isoclines are curves along which the directions of $X$ are parallel transports of each other.
  As in the Euclidean case, generalized isoclines of generic vector fields $X$ are smooth curves that connect equilibria of $X$.
  Our algorithm can be regarded as an extension of the method of Newton trajectories to the manifold setting when the manifold is unknown.
  \\

  This work is motivated by computational statistical mechanics, specifically high dimensional (stochastic) differential equations that model the dynamics of molecular systems.
  Often, these dynamics concentrate near low-dimensional manifolds and have transitions (saddle points with a single unstable direction) between metastable equilibria.
  We employ iteratively sampled data and isoclines to locate these saddle points.
  Coupling a black-box sampling scheme (\emph{e.g.,} Markov chain Monte Carlo) with manifold learning techniques (diffusion maps in the case presented here), we show that our method reliably locates equilibria of $X$.
}


%% file: introduction.tex
\section{Introduction}

This paper presents a method for locating equilibria of dynamical systems on manifolds.
In the applications we envision, there is a compact low-dimensional submanifold $M$ of a high-dimensional Euclidean space and a smooth vector field $X$ on $M$.
By an equilibrium point (also known as a fixed point) of $X$ we mean a point $p \in M$ such that $X(p) = 0$.
We assume that the equilibria of $X$ are isolated and do not lie at the boundary of $M$.
The manifold $M$ is \emph{a priori} unknown and, by assumption, we are capable of randomly sampling the vicinity of arbitrary points of it.
This is a reasonable assumption in cases of practical interest.
Importantly, $M$ is \emph{not} required to be defined either explicitly via an atlas or implicitly by the zeros of a smooth map.
Instead, we apply manifold learning techniques to iteratively construct an atlas, sampling new points as needed.

Possible uses of our method include the exploration of energy landscapes in computational statistical mechanics~\cite{mezey1987,bryngelson1987,wales2018} and energy-based models in deep learning~\cite{lecun2006,pascanu2014,li2018,lou2020}, as well as the identification of transition states in chemical systems~\cite{bochevarov2013,jackson2021}.
There are several additional settings in which dynamical systems on manifolds arise naturally: namely, as exterior differential systems in differential geometry~\cite{bryant1991}, as mechanical systems with holonomic constraints~\cite{fixman1974,arnold1989,mclachlan2014,hairer2006}, as inertial manifolds~\cite{constantin1989}, and as the result of the concentration of measure phenomenon~\cite{wainwright2019}.

In computational statistical mechanics, the state of a mechanical system (\emph{e.g.,} a protein solvated in water) is a point in an Euclidean space of $n = 6 N - 6$ dimensions (\emph{i.e.,} the aggregated three-dimensional positions and momenta of the $N$ atoms comprising the system minus three rotational and three translational degrees of freedom) while the progress of the chemical process being simulated is often governed by a small set of, say, $m \ll n$ degrees of freedom known as collective variables (\emph{e.g.,} dihedral angles, radius of gyration, end-to-end distance, etc.).
The effective motion of the system is driven by the potential of mean force~\cite{roux1995} and it is described by the time-evolution of a point in the $m$-dimensional space of collective variables $M$.
Often, the relevant collective variables are not known in advance and practitioners rely on machine learning for identifying them~\cite{coifman2008,georgiou2017,chen2019,sidky2020,tsai2021}.

Finding transition states (saddle equilibria) between stable equilibria on the vector field $X$ on $M$ determined by the negative of the gradient of the potential of mean force~\cite{kirkwood1935} is useful for the efficient simulation of molecular systems since, under the dynamics of $X$, most of the simulation time is spent in the basin of attraction of a stable equilibrium and only rarely does the system visit another basin of attraction, typically by passing close to a saddle point en route~\cite{varadhan2016}.
Thus, direct searches for transition states help accelerate the sampling of the phase space of the system.

Multiple strategies have been developed over the years for locating saddle points.
In one class, the vector field $X$ is modified to ensure that a simulation exhaustively explores the conformational state space of the system (this is the case of adaptive biasing force~\cite{darve2008}, metadynamics~\cite{laio2002}, and iMapD~\cite{chiavazzo2017}).
Another class of methods constructs a path in a space of collective variables joining two minima of the potential of mean force by locally minimizing the free energy along its points.
The technique was pioneered by~\cite{ulitsky1990} and many variants exist under the name of the string method~\cite{e2002,peters2004,pan2008}.
Yet another class is comprised by eigenvector-following methods such as~\cite{crippen1971,cerjan1981,lucia2002}, the dimer method~\cite{henkelman1999}, and gentlest ascent dynamics~\cite{e2011,gu2017}.
Eigenvector\--following methods only require the coordinates of a free energy minimum as input and they attempt to produce a path joining the input point to a saddle point.
However, some of these methods may not converge globally~\cite{levitt2017}.
The last class of methods includes approaches such as reduced gradient following / Newton trajectories~\cite{quapp1998,hirsch2004} which require a single initial point as input and follow a path leading to an equilibrium (often a saddle), even if the resulting paths may be unphysical.

Our method extends the Newton trajectory concept in two crucial ways: (a) from the Euclidean to the Riemannian manifold setting and, importantly, (b) to the case where the collective variables are not known in advance, but are automatically and gradually discovered using manifold learning as part of the algorithm.
The use of Newton trajectories in non-Euclidean settings was initiated in~\cite{quapp2004} for the special case of the manifold of internal coordinates of a molecule (considering a single chart), adopting an extrinsic approach to the geometrical problem, and assuming \emph{a priori} knowledge of the collective variables.
In our variant of the method, only one initial point is required and there is no need for prior knowledge of a set of collective variables of the system.

The paper is organized as follows:
\Cref{sec:isoclines} introduces the notion of isoclines in Euclidean spaces and generalized isoclines in Riemannian manifolds, establishing the results on which the algorithm is based.
\Cref{sec:numerical-scheme} presents our method in detail and illustrates it with simple examples.
\Cref{sec:conclusion} revisits our work and lays out additional avenues of research.


%% file: isoclines.tex
\section{Generalized isoclines}
\label{sec:isoclines}

In this section we show that in any sufficiently small neighborhood $U$ of an isolated equilibrium of a smooth vector field $X$ defined on a connected, compact Riemannian manifold $M$, one can find trajectories whose tangent at a point $x \in U$ are parallel to an arbitrary fixed direction $V \in T_x M$.
This implies that a curve $\gamma$ such that the vector field $X$ on the curve is parallel to $V$ must necessarily connect equilibria of $X$.

We first introduce the notion of an \emph{isocline} in the Euclidean setting.
Then, we discuss the framework under which one can compute such isoclines.
After that, we generalize our results to the case where $M$ is an $m$-dimensional Riemannian manifold.

\subsection{Isoclines in Euclidean space}

Consider the two-dimensional dynamical system given by
\begin{equation}
  \label{eq:ode-2d}
  \left\{
    \begin{aligned}
      &\frac{\d x}{\d t} = f(x, y), \\
      &\frac{\d y}{\d t} = g(x, y), \\
    \end{aligned}
  \right.
\end{equation}
for some smooth real-valued functions $f, g \in C^\infty(\R^2)$.
An isocline~\cite{guckenheimer1983} of~\eqref{eq:ode-2d} is a curve $\gamma$ whose trace is the set of points $(x, y) \in \R^2$ such that
\begin{equation*}
  \frac{\d y}{\d x} = \frac{g(x, y)}{f(x, y)} \equiv \text{constant},
\end{equation*}
or, equivalently, in the case when both $f(x, y) \neq 0$ and $g(x, y) \neq 0$, we get
\begin{equation}
  \label{eq:isocline-2d}
  \frac{\d x}{f(x, y)} = \frac{\d y}{g(x, y)}.
\end{equation}

Now consider the Euclidean space $M = \R^m$ and let us denote by $\X(M)$ the set of smooth vector fields defined on $M$.
Let $X \in \X(M)$ be the vector field associated with the dynamical system $\dot{x} = X(x)$.
In coordinates, we write $X(x) = (X^1(x), \dotsc, X^m(x))$.
The characteristic system of ordinary differential equations \eqref{eq:isocline-2d} can be immediately generalized to
\begin{equation*}
  \frac{\d x^1}{X^1(x)}
  =
  \dotsb
  =
  \frac{\d x^m}{X^m(x)}
\end{equation*}
for the vector field $X$ and is valid at points such that $X^i(x) \neq 0$ for all $i = 1, \dotsc, m$.

\begin{definition}[Isocline]
  Let $Y = X / \| X \|$ be the normalized vector field corresponding to $X \in \X(M)$.
  An \emph{isocline} (see \Cref{fig:isodirectional-curve}) of the vector field $X$ in the direction $V \in \R^m$ with $\| V \| = 1$ is a curve $\gamma \colon [0, T] \to M$ satisfying
  \begin{equation}
    \label{eq:isocline}
    Y(\gamma(s)) = V
    \quad \text{for} \quad
    0 \le s \le T.
  \end{equation}
\end{definition}

\begin{figure}[ht]
  \centering
  {
    \includegraphics{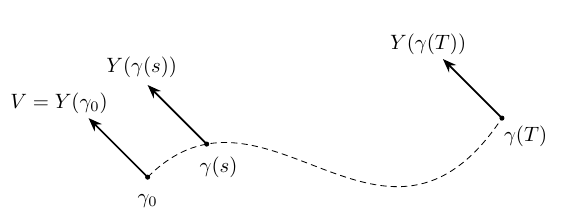}
    %
    %
    %
    %
    %
    %
    %
    %
  }
  \caption{Isocline $\gamma$ (dashed line) and normalized vector field $Y$ (arrows) at several points of the curve.}
  \label{fig:isodirectional-curve}
\end{figure}

\begin{remark}
  There are vector fields for which the definition of isocline is hardly useful.
  For example, a constant vector field $X(x) \equiv V$ has isoclines only in the direction $V$ and all curves are isoclines for this direction.
  Thus, the definition is most useful when $Y$, regarded as a \emph{map} from the complement of the equilibria of $X$ to the unit sphere is onto and almost all points of the sphere are regular values of this map.
  The implicit function theorem implies that the isoclines of regular values are smooth curves.
  Throughout the paper, we assume implicitly that we are computing isoclines of regular values of $Y$.
  Investigation of the isoclines of singular values of $Y$ for generic vector fields $X$ is an interesting topic for further work.
\end{remark}

\begin{proposition}
  \label{thm:all-directions}
  Let $x_\star \in \R^m$ be an isolated equilibrium of the dynamical system $\dot{x} = X(x)$.
  Let $V \in \R^m$ such that $\| V \| = 1$.
  Then, there exists a neighborhood $U \subset \R^m$ of $x_\star$ and points $x \in U$ such that $X(x) = \lambda V$ for some $\lambda \in \R \setminus \{ 0 \}$.
  Moreover, the set of such $x$ forms a curve (see also~\cref{fig:directions}).
\end{proposition}

\begin{figure}[ht]
  \centering
  \begin{subfigure}[b]{0.48\columnwidth}
    \includegraphics[width=0.75\columnwidth]{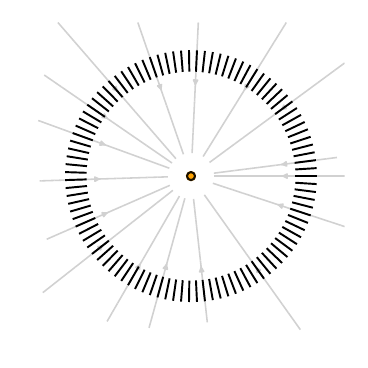}
    \caption{Sink (\textcolor{orange}{$\bullet$}).}
  \end{subfigure}
  \hfill
  \begin{subfigure}[b]{0.48\columnwidth}
    \includegraphics[width=0.75\columnwidth]{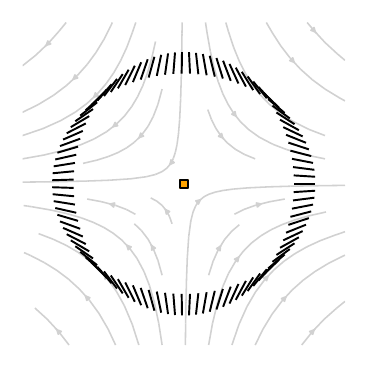}
    \caption{Saddle point (\textcolor{orange}{$\blacksquare$}).}
  \end{subfigure}
  \caption{Directions (black line segments) around two equilibria (marked in orange) of a vector field (gray stream lines).
    In both cases, we see that there exists an orbit in the neighborhood of the equilibrium that is tangent at some point to any prescribed direction.}
  \label{fig:directions}
\end{figure}

\begin{example}
  A \emph{nullcline} is an isocline whose direction is of the form $V = e_i$, where $\{ e_\ell \mid \ell = 1, \dotsc, m \}$ is the canonical basis of $\R^m$.
  Indeed, the traces of the nullclines of $\dot{x} = X(x)$ are, respectively, the sets
  \begin{equation*}
    N_i = \{ x \in \R^m \mid X^i(x) = 0 \}.
  \end{equation*}
  Clearly, the intersection $\cap_{i = 1}^m N_i$ is the set of equilibria of the vector field $X$.
\end{example}

Mapping tangent spaces between points on a Riemannian manifold is more subtle than translation of vectors in $\R^m$.
Differential geometry introduces the concept of affine connections which give a way to relate tangent spaces.
To exploit this framework, we reformulate the definition of isocline.
Differentiating~\eqref{eq:isocline} yields
\begin{equation*}
  \left\{
    \begin{aligned}
      &\frac{\d Y}{\d s}(\gamma(s)) = 0, \\
      &Y(\gamma_0) = V,
    \end{aligned}
  \right.
\end{equation*}
where $\gamma_0 = \gamma(0) \in \R^m$.
Moreover, since
\begin{equation*}
  \frac{\d}{\d s} Y(\gamma(s)) = D Y(\gamma(s)) \, \dot{\gamma}(s),
\end{equation*}
we can rewrite~\eqref{eq:isocline} as
\begin{equation}
  \label{eq:euclidean-parallel-transport}
  \left\{
    \begin{aligned}
      &D Y(\gamma(s)) \, \dot{\gamma}(s) = 0, \\
      &Y(\gamma_0) = V.
    \end{aligned}
  \right.
\end{equation}

We will next see that the isocline $\gamma \colon [0, T] \to M$ in~\eqref{eq:euclidean-parallel-transport} can be solved for, as long as $Y = X / \| X \|$ is known, $\gamma_0 \in M$ is an arbitrary phase space point, and $V = Y(\gamma_0)$.

\begin{proposition}
  \label{thm:euclidean-field-of-lines}
  Let $X$ be a smooth vector field in $\R^m$ and let $Y = X / \| X \|$ be defined on an open set $U \subset \R^m$ on which $X$ does not vanish.
  If the Jacobian matrix $D X$ of $X$ is full rank in $U$, then the nullspace of the Jacobian matrix $D Y$ of $Y$ is one-dimensional in $U$, giving rise to a unique field of lines.
\end{proposition}
\begin{proof}
  The matrix $D Y$ is proportional to the orthogonal projection $Q$ of the column space of $D X$ onto the orthogonal complement of $X$.
  Indeed,
  \begin{equation*}
    D Y = \frac{1}{\| X \|} Q \, D X,
    \quad \text{where} \quad
    Q = I - (X^\top X)^{-1} X X^\top.
  \end{equation*}
  Since $\ker Q = \Span\{ X \}$ and $D X$ is full rank by hypothesis, we conclude that $\ker D Y = \Span\{ L \}$ where $L = \adj(D X) \, X$ with $\adj(A)$ denoting the transpose matrix of cofactors (or \emph{adjugate}) of $A$.
  Recall that $A^{-1} = \det(A)^{-1} \adj(A)$, so even if $D X$ is singular, $L$ is well defined.
\end{proof}

The problem~\eqref{eq:euclidean-parallel-transport} is that of finding an integral curve tangent to the unique field of lines of Proposition~\ref{thm:euclidean-field-of-lines}.
We have formulated the problem of obtaining isoclines as a system of differential-algebraic equations because it lends itself to generalization to the Riemannian setting.
The usual approach in Euclidean space is to compute isoclines using numerical continuation methods.

\subsection{Generalized Isoclines on Riemannian manifolds}
\label{sec:riemannian-isoclines}

Let $M$ be a smooth, connected, and compact $m$-dimensional manifold with Riemannian metric $g$.
We denote the set of smooth, real-valued functions on $M$ by $C^\infty(M)$.

With modest restrictions on the smooth vector field $X \in \X(M)$, we generalize the notion of an isocline to this setting and then use the generalized isoclines as curves joining equilibria of $X$.
The key idea is that the condition $\frac{\d}{\d s} Y(\gamma(s)) = D Y(\gamma(s)) \dot{\gamma}(s) = 0$ is replaced by $\nabla_{\dot \gamma(s)} Y(\gamma(s)) = 0$, where $\nabla$ in this case is the covariant derivative determined by an affine connection on $M$ and the normalized vector field is now given by $Y = X / \sqrt{g(X, X)} \in \X(M)$.

Consider an arbitrary chart $(U, \phi)$ where $U \subset M$ is an open set and
\begin{equation*}
  \phi = (\phi^1, \dotsc, \phi^m) \colon U \to \R^m
\end{equation*}
is a system of coordinates with corresponding parameterization $\psi \colon \R^m \to U \subset M$ such that $\psi = \phi^{-1}$ (see Figure~\ref{fig:chart}).
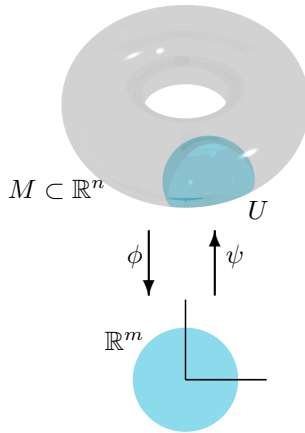
\begin{figure}[ht]
  \centering
  \begin{overpic}[width=0.325\columnwidth]{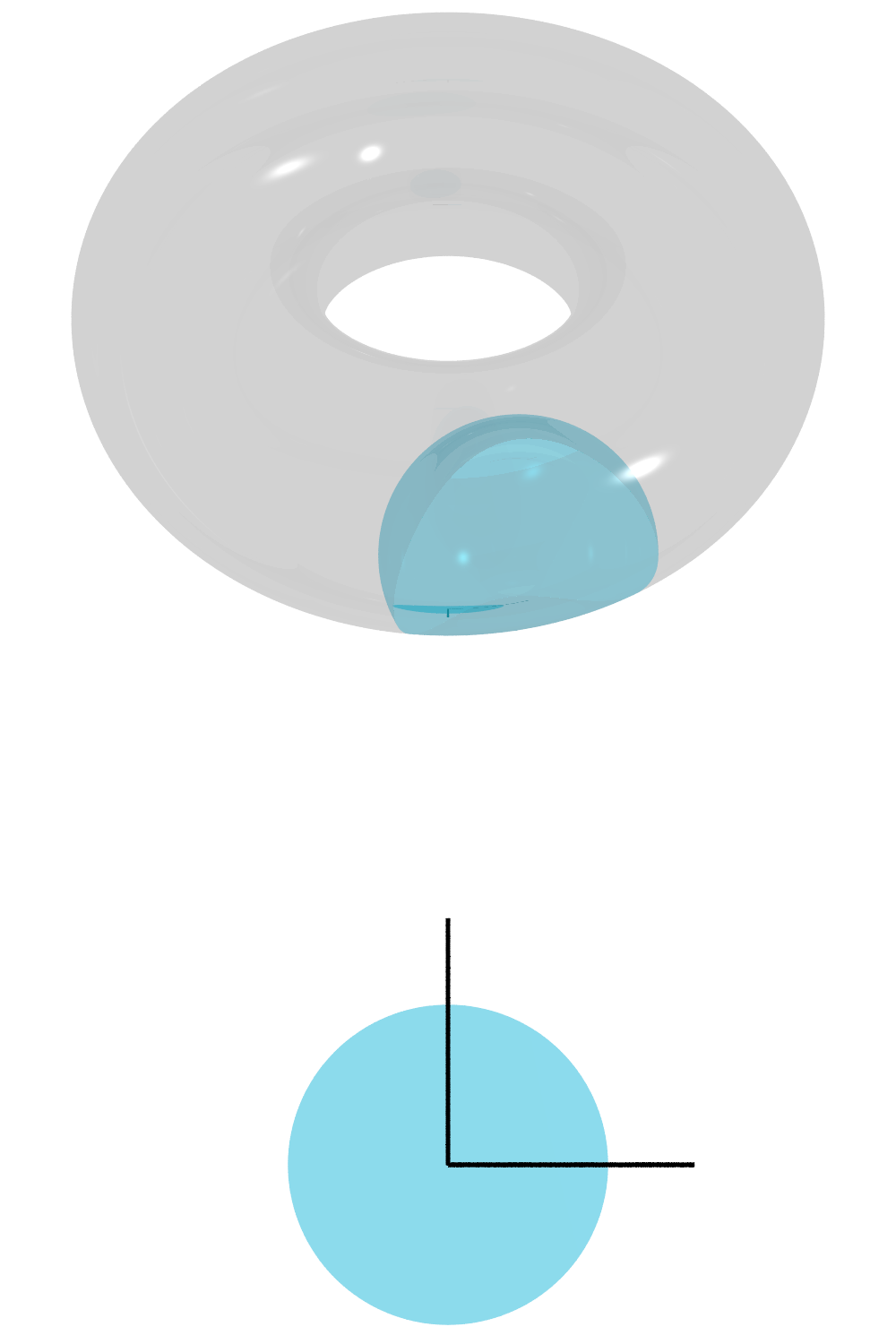}
    \thicklines
    \put(25, 47.5){\vector(0, -1){15}}
    \put(40, 32.5){\vector(0, 1){15}}
    \put(20, 40){$\phi$}
    \put(42.5, 40){$\psi$}
    \put(-7, 55){$M \subset \R^n$}
    \put(47.5, 50){$U$}
    \put(15, 20){$\R^m$}
  \end{overpic}
  \caption{System of coordinates $\phi$ and parameterization $\psi$ on an open set $U$ of $M$.}
  \label{fig:chart}
\end{figure}
Let $\{ \partial_i = \frac{\partial}{\partial \phi^i} \mid i = 1, \dotsc, m \}$ be a basis of the tangent space $T_x M$ at a point $x \in M$.
In local coordinates, we write $X = \sum_{i = 1}^m X^i \, \partial_i$ and $Y = \sum_{i = 1}^m Y^i \, \partial_i$.

To carry out (covariant) differentiation on a Riemannian manifold, we  introduce an affine connection~\cite[Chapter 2]{docarmo1992}.
For the purposes of this paper, an affine connection is a mapping $\nabla \colon \X(M) \times \X(M) \to \X(M)$ that is $C^\infty(M)$-linear in its first argument and it is a derivation~\cite{bourbaki2009} in its second argument.
That is, for $f, g \in C^\infty(M)$, $\alpha, \beta \in \R$, and $V, W, Z \in \X(M)$, we have
\begin{enumerate}
\item $\nabla_{f Z + g V} W = f \nabla_Z W + g \nabla_V W$,
\item $\nabla_Z (\alpha V + \beta W) = \alpha \nabla_Z V + \beta \nabla_Z W$,
and
\item $\nabla_Z (f V) = (Z f) V + f \nabla_Z V$.
\end{enumerate}
In coordinates, a connection is characterized by
\begin{equation*}
  \nabla_{\partial_i} \partial_j = \sum_{k = 1}^m \Gamma_{ij}^k \partial_k
\end{equation*}
for $i, j = 1, \dotsc, m$, where $\Gamma_{ij}^k \in C^\infty(M)$.
The coefficients $\Gamma_{ij}^k$ are called \emph{Christoffel symbols} when the affine connection is the Levi-Civita connection~\cite{docarmo1992}.
In this work, we use the Levi-Civita connection.

The parallel transport equation
\begin{equation}
  \label{eq:riemannian-parallel-transport}
  \nabla_{\dot{\gamma}} Y
  =
  \sum_{k = 1}^m \left\{
    \sum_{i = 1}^m
    \left(
      \frac{\partial Y^k}{\partial \phi^i}
      +
      \sum_{j = 1}^m \Gamma_{ij}^k Y^j
    \right)
    \dot{\gamma}^i
  \right\}
  \partial_k
  =
  0.
\end{equation}
coincides with~\eqref{eq:euclidean-parallel-transport} when $Y(\gamma_0) = V$ and $M = \R^m$ (\emph{i.e.,} the covariant derivative of the flat connection $\bar{\nabla}$ is just the ordinary derivative).
Indeed, the relation between~\eqref{eq:euclidean-parallel-transport} and~\eqref{eq:riemannian-parallel-transport} becomes clearer when we write the two equations side by side as
\begin{equation*}
  \left\{
    \begin{aligned}
      &\bar{\nabla}_{\dot{\gamma}(s)} Y(\gamma(s)) = D Y(\gamma(s)) \, \dot{\gamma}(s) = 0, \\
      &Y(\gamma_0) = V,
    \end{aligned}
  \right.
  \qquad
  \left\{
    \begin{aligned}
      &\nabla_{\dot{\gamma}(s)} Y(\gamma(s))
      =
      0, \\
      &Y(\gamma_0) = V.
    \end{aligned}
  \right.
\end{equation*}

Geometrically, equation~\eqref{eq:riemannian-parallel-transport} is satisfied when $Y(\gamma(s))$ is equal to the parallel translation of $Y(\gamma(s + \tau))$ as $\tau \to 0$ for $s \in [0, T]$ (see \Cref{fig:geometric-construction-covariant-derivative} and \Cref{sec:three-views-parallel-transport}).

\begin{figure}[ht]
  \centering
  \includegraphics{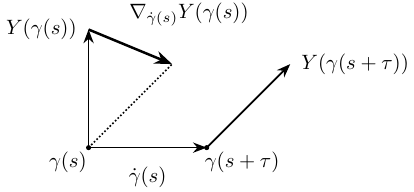}
  %
  %
  %
  %
  %
  %
  %

  \caption{For fixed $s \in [0, T]$, the vector $\dot{\gamma}(s)$ has its tip located at $\gamma(s + \tau)$, as $\tau \to 0$, and its tail is at $\gamma(s)$.
    The covariant derivative $\nabla_{\dot{\gamma}(s)} Y(\gamma(s))$ is the vector joining the tips of the vector $Y(\gamma(s))$ and the parallel translation (represented as a dotted segment) of the vector $Y(\gamma(s + \tau))$ to the point $\gamma(s)$.
    Consequently, if a curve $\gamma$ is such that $\nabla_{\dot{\gamma}} Y(\gamma) = 0$, then $Y(\gamma(s))$ coincides with the parallel translation of $Y(\gamma(s + \tau))$ to $\gamma(s)$ as $\tau \to 0$ for all $s \in [0, T]$.
  }
  \label{fig:geometric-construction-covariant-derivative}
\end{figure}

For a given $x \in M$, we can keep the second argument $W \in \X(M)$ of $\nabla$ fixed to get the linear mapping $\nabla W \colon T_x M \to T_x M$.
In coordinates, this mapping is
\begin{equation}
  \label{eq:nabla}
  \nabla W
  =
  \sum_{k = 1}^m
  \sum_{i = 1}^m
  a_i^k(W)
  \,
  \frac{\partial}{\partial \phi^k} \otimes \d \phi^i
  \quad \text{where} \quad
  a_i^k(W)
  =
  \sum_{j = 1}^m
  \left(
    \frac{\partial W^k}{\partial \phi^j}
    +
    \Gamma_{ij}^k W^j
  \right).
\end{equation}
We represent $\nabla W$ by the $m \times m$ matrix $A(W)$ with components $a_i^k(W)$.

\begin{proposition}
  \label{thm:riemannian-field-of-lines}
  Let $U$ be an open set of $M$ not containing equilibria of $X \in \X(M)$ and let $Y = X / \sqrt{g(X, X)} \in \X(M)$.
  If $\nabla X$ at a point $x \in U$ is a linear automorphism, then the nullspace of $\nabla Y$ at $x$ is one-dimensional.
\end{proposition}
\begin{proof}
  Let $x \in U$.
  Without loss of generality, we consider geodesic normal coordinates $\phi = (\phi^1, \dotsc, \phi^m)$ around a neighborhood $W \subseteq U$ of $x$.
  This implies that at $x$, we have
  \begin{equation*}
    g_{ij} = \delta_{ij},
    \quad
    \Gamma_{ij}^k = 0,
    \quad
    \nabla_{\partial_j} \partial_i = 0,
    \quad
    \text{and}
    \quad
    \frac{\partial g_{ij}}{\partial \phi^k} = 0,
  \end{equation*}
  for $i, j, k = 1, \dotsc, m$.
  Consequently, the components of $\nabla Y$ are determined by
  \begin{align*}
    \nabla_{\partial_j} (Y^i \partial_i)
    &=
      \nabla_{\partial_j} ( X^i/ \sqrt{g(X, X)} \, \partial_i )
      =
      \frac{\partial}{\partial \phi^j} \left( X^i / \sqrt{g(X, X)} \right) \\
    &=
      \frac{1}{\sqrt{g(X, X)}}
      \left(
      \frac{\partial X^i}{\partial \phi^j}
      -
      \frac{1}{g(X, X)} \sum_{k = 1}^m X^i X^k \frac{\partial X^k}{\partial \phi^j}
      \right).
  \end{align*}
  In matrix notation, the above becomes
  \begin{equation*}
    A(Y)
    =
    (X^\top X)^{-1/2} \, Q \, A(X),
  \end{equation*}
  where $Q = I - (X^\top X)^{-1} X X^\top$ is the orthogonal projection matrix onto the complement of $X$ in $T_x M$.
  Since $\ker Q = \Span\{ X \}$, and $A(X)$ is full-rank by hypothesis, we conclude that $\ker A(Y) = \Span\{ \adj(A(X)) X \}$.
\end{proof}

Choosing an orientation for a unit vector in $\ker A(Y)$ yields a unit length vector field for the isoclines of $X$ on open sets where $A(X)$ has full rank. Numerical integration of this equation produces the generalized isoclines.

\begin{remark}
  Solving the differential equations \eqref{eq:euclidean-parallel-transport} and \eqref{eq:riemannian-parallel-transport} give analogous methods for finding isoclines in the Euclidean space and Riemannian manifold settings.
  However, in Euclidean space one can do more, since isoclines satisfy the algebraic equations that impose a constant direction $V$.
  Numerically, one can apply continuation methods that alternate predictor and corrector steps to find isoclines.
  After each step of an initial value solver applied to \eqref{eq:euclidean-parallel-transport}, an equation solver can be used to reduce the distance of the computed point to the isocline.
  In the Riemannian setting, path dependence of parallel transport precludes corrector steps.
  However, we can choose our time steps adaptively to reduce the error.
\end{remark}

\begin{example}[Sphere]
  \label{ex:sphere}
  The M\"uller-Brown potential~\cite{muller1979}, shown in \Cref{fig:muller-brown-potential}, is given by
  \begin{equation}
    \label{eq:muller-brown}
    U(y^1, y^2)
    =
    \sum_{i = 1}^4 A_i \, \e^{a_i (y^1 - y^1_{0i})^2 + b_i (y^1 - y^1_{0i}) (y^2 - y^2_{0i}) + c_i (y^2 - y^2_{0i})^2},
  \end{equation}
  where we use the coefficients listed in \Cref{tab:muller-brown}.
  \begin{table}[ht]
    \centering
    \begin{tabular}{c|rrrr}
      $i$ & 1 & 2 & 3 & 4 \\
      \hline
      $A_i$ & $-200$ & $-100$ & $-170$ & $15$ \\
      $a_i$ & $-1$ & $-1$ & $-6.5$ & $0.7$ \\
      $b_i$ & $0$ & $0$ & $11$ & $0.6$ \\
      $c_i$ & $-10$ & $-10$ & $-6.5$ & $0.7$ \\
      $y^1_{0i}$ & $1$ & $0$ & $-0.5$ & $-1$ \\
      $y^2_{0i}$ & $0$ & $0.5$ & $1.5$ & $1$ \\
    \end{tabular}
    \caption{Parameters of the (planar) M\"uller-Brown potential in Cartesian coordinates.}
    \label{tab:muller-brown}
  \end{table}
  This is a typical model system used in computational chemistry for a chemical reaction that progresses along a curve in a two-dimensional space of collective variables (see \Cref{fig:muller-brown-potential}).
  \begin{figure}[ht]
    \centering
    \includegraphics[width=0.6\columnwidth,trim=1.25cm 0.25cm 0.5cm 0.125cm,clip]{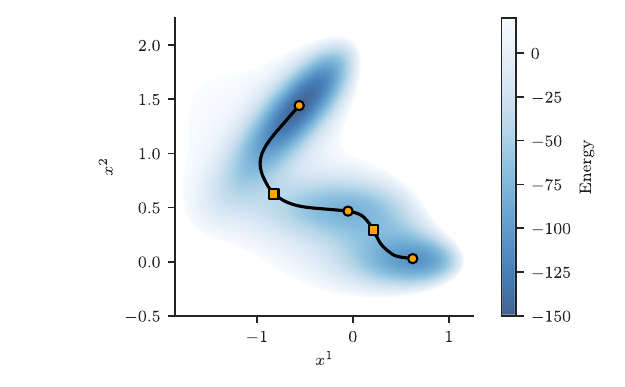}
    \caption{M\"uller-Brown potential (heat map) and path of least action~\cite{arnold1989} (black curve) joining all equilibria. Sinks are represented by (\textcolor{orange}{$\bullet$}) and saddles are represented by (\textcolor{orange}{$\blacksquare$}).}
    \label{fig:muller-brown-potential}
  \end{figure}
  The saddles and sinks of $U$ are listed in \Cref{tab:muller-brown-equilibria}.
  \begin{table}[ht]
    \centering
    \begin{tabular}{l|l|l}
      $y^1$ & $y^2$ & Type \\
      \hline
      $\phantom{-}0.623499404930877$   & $0.0280377585286857$ & sink \\
      $\phantom{-}0.212486582000662$   & $0.292988325107368$  & saddle \\
      $-0.0500108229982061$            & $0.466694104871972$  & sink \\
      $-0.822001558732732$             & $0.624312802814871$  & saddle \\
      $-0.558223634633024$             & $1.44172584180467$   & sink
    \end{tabular}
    \caption{Coordinates of the equilibrium points of the M\"{u}ller-Brown potential and their type.}
    \label{tab:muller-brown-equilibria}
  \end{table}

  Since we are concerned with finding transition states when the conformational phase space of our system is a compact Riemannian manifold, we define the M\"uller-Brown potential on the sphere $\S^2 = \{ (x^1, x^2, x^3) \in \R^3 \mid (x^1)^2 + (x^2)^2 +(x^3)^2 = 1 \}$ as a composition of transformations $E = U \circ \kappa \circ \xi$, where
  \begin{equation*}
    \xi(x^1, x^2, x^3) = (\arctan(x^2 / \, x^1), \arctan(x^3 / \sqrt{(x^1)^2 + (x^2)^2})),
  \end{equation*}
  is a coordinate change from Cartesian to spherical polar coordinates,
  \begin{equation*}
    \kappa(k^1, k^2) = (1.973521294 \, k^1 - 1.85, 1.750704373 \, k^2 + 0.875),
  \end{equation*}
  is an affine mapping that sends the polar coordinates $(k^1, k^2) \in [0, \frac{\pi}{2}] \times [\frac{\pi}{4}, \frac{3 \pi}{4}]$ to the rectangle $R = [-1.85, \frac{5}{4}] \times [-\frac{1}{2}, \frac{9}{4}]$, and $U$ is the planar M\"uller-Brown potential~\eqref{eq:muller-brown}.
  Observe that $R$ contains all the equilibria of $U$ (as shown in \Cref{fig:muller-brown-potential}) and its complement is a highly energetic region of phase space.

  \Cref{fig:muller-isocline-field} shows the M\"uller-Brown potential $E$ on the sphere as well as in one chart of the stereographic projection (see \Cref{sec:stereographic-projection}).
  Note that in this particular realization, the generalized isocline shown in the figure traverses all the equilibria of the vector field on the sphere.

  \begin{figure}[ht]
    \centering
    \begin{subfigure}[t]{0.475\columnwidth}
      \includegraphics[width=\columnwidth,trim=0.3cm 0cm 0.5cm 1cm,clip]{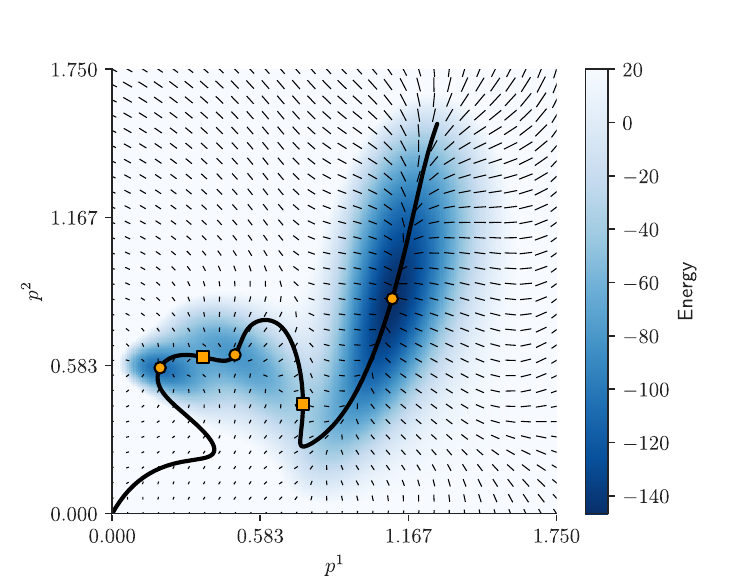}
      \caption{One isocline (black line) and potential energy function (blue heat map) of the M\"uller-Brown potential on the stereographic projection from the North pole of $\S^2$ onto the tangent plane to the South pole.
        The segments are tangent to the gradient field of the M\"uller-Brown potential.}
      \label{fig:muller-isocline-field-a}
    \end{subfigure}
    \hfill
    \begin{subfigure}[t]{0.475\columnwidth}
      \includegraphics[width=\columnwidth]{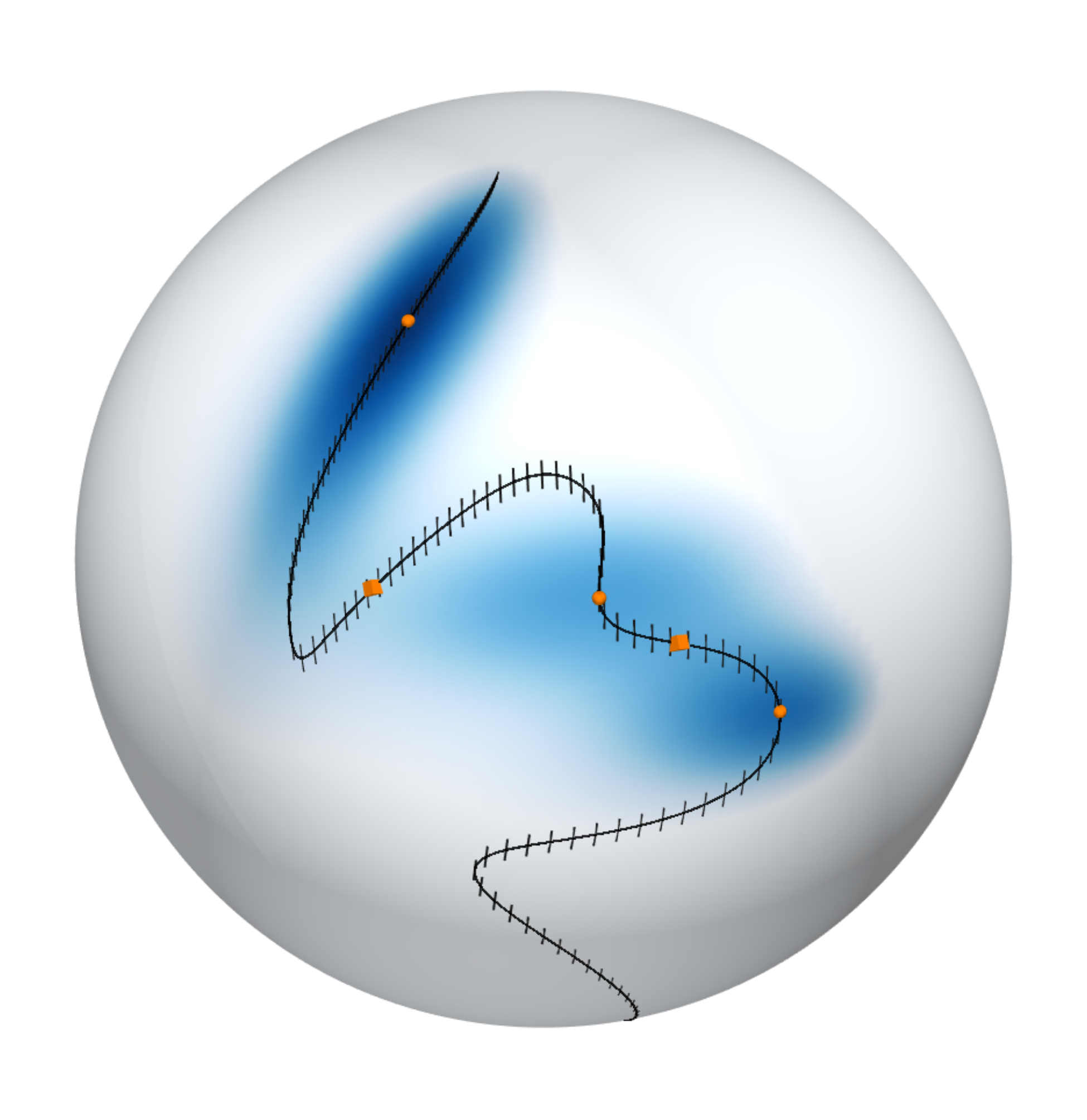}
      \caption{Isocline (black line) on the sphere $\S^2 \subset \R^3$.
        The black segments are tangent to the gradient field of the M\"uller-Brown potential on the sphere.}
      \label{fig:muller-isocline-field-b}
    \end{subfigure}
    \caption{Equilibria obtained by staying the course on the M\"uller-Brown potential on a sphere.}
    \label{fig:muller-isocline-field}
  \end{figure}
\end{example}

\begin{example}[Pseudosphere]
  While the sphere $\S^2$ has constant sectional curvature equal to $+1$, the pseudosphere $\mathbb{P}$~\cite{milnor1973} is a Riemannian manifold with constant sectional curvature equal to $-1$.
  The pseudosphere is given by the zeros of
  \begin{equation*}
    (x^3)^2 = \left(\arcsech \sqrt{(x^1)^2 + (x^2)^2} - \sqrt{1 - (x^1)^2 - (x^2)^2} \right)^2.
  \end{equation*}
  We parameterize the upper-half of $\mathbb{P} \subset \R^3$ by
  \begin{equation*}
    \chi(k^1, k^2)
    =
    (k^1 \cos k^2, k^1 \sin k^2, \arcsech k^1 - \sqrt{1 - (k^1)^2}).
  \end{equation*}
  Next, we write the potential as $E = U \circ \kappa \circ \xi$, where
  \begin{equation*}
    \xi(x^1, x^2, x^3)
    =
    \left(
      \sqrt{(x^1)^2 + (x^2)^2}, \arctan({x^2} / \, {x^1})
    \right),
  \end{equation*}
  is the inverse of the parameterization $\chi$,
  \begin{equation*}
    \kappa(k^1, k^2) = (0.9867606472 \, k^2 - 1.85, 4.406507321 \, k^1 - 1.715856588)
  \end{equation*}
  is an affine mapping from $[0, \pi] \times [0.2756, 0.9]$ to $R$, and $U$ is, again, \eqref{eq:muller-brown}.

Despite the fact that the pseudosphere $\mathbb{P}$ is not compact, it nevertheless exhibits curves $\gamma$ satisfying $\nabla_{\dot{\gamma}} Y = 0$ that connect equilibria of $X$ (see \Cref{fig:muller-pseudosphere}).

  \begin{figure}[ht]
    \centering
    \begin{subfigure}[t]{0.48\columnwidth}
      \includegraphics[width=\columnwidth,trim=0cm 0.25cm 0cm 0cm,clip]{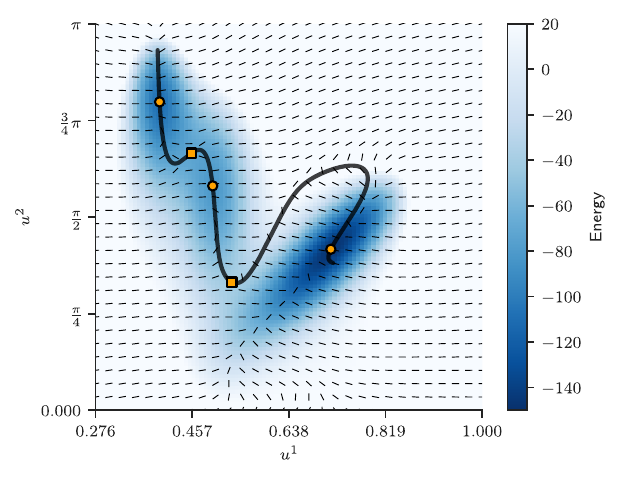}
      \caption{One isocline (black line) and potential energy function (blue heat map) of the M\"uller-Brown potential on the pseudosphere.
        The segments are tangent to the gradient field of the M\"uller-Brown potential.}
      \label{fig:muller-pseudosphere-a}
    \end{subfigure}
    \hfill
    \begin{subfigure}[t]{0.475\columnwidth}
      \includegraphics[width=\columnwidth,trim=0cm 3.5cm 0cm 0cm,clip]{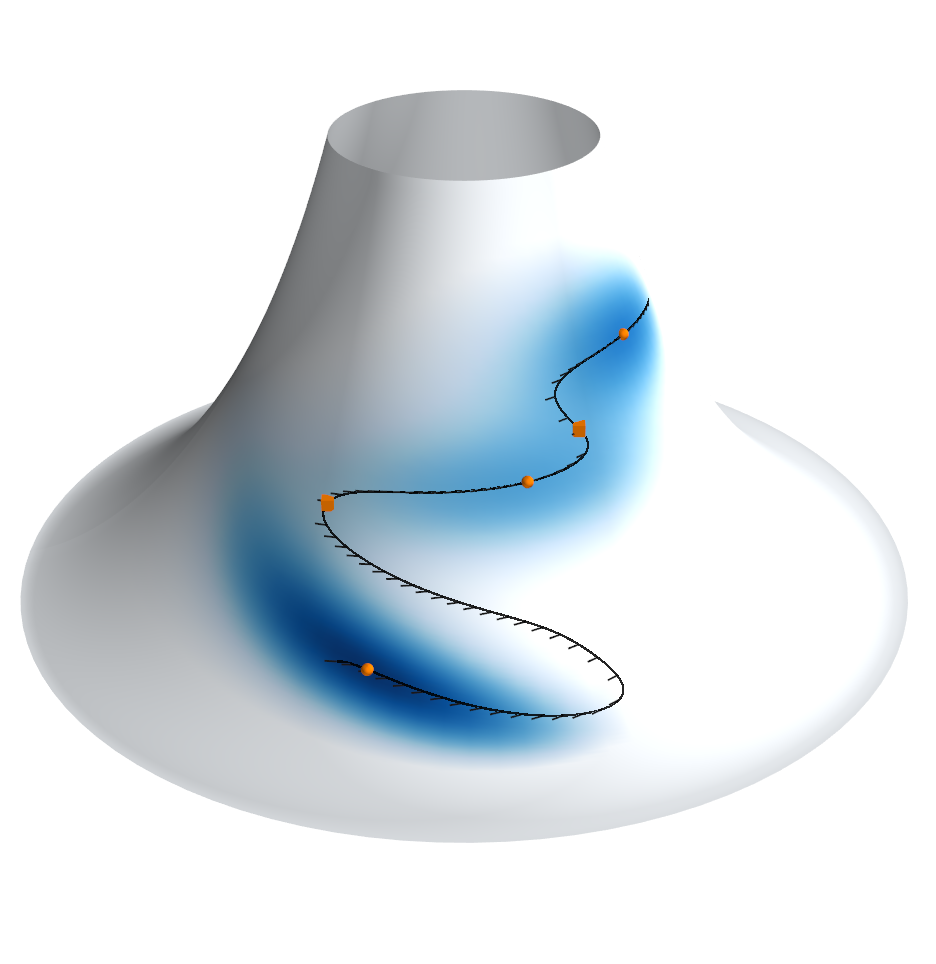}
      \caption{Isocline (black line) on the pseudosphere $\mathbb{P} \subset \R^3$.
        The black segments are tangent to the gradient field of the M\"uller-Brown potential on the pseudosphere.}
      \label{fig:muller-pseudosphere-b}
    \end{subfigure}
    \caption{Equilibria obtained by staying the course on the M\"uller-Brown potential on a pseudosphere.}
    \label{fig:muller-pseudosphere}
  \end{figure}
\end{example}


%% file: algorithm.tex
\section{Numerical scheme and examples}
\label{sec:numerical-scheme}

Let $M$ be a compact Riemannian manifold and let $X \in \X(M)$ be a vector field on $M$.
\Cref{alg:staying-the-course} below traces a curve $\gamma$ that starts at an arbitrary regular point $\gamma_0 \in M$ of $X$ and ends at an isolated equilibrium $x_\star$ of $X$.
The curve $\gamma$ traced by the algorithm is tangent at every point to the vector field $\dot{\gamma} \in \X(M)$ defined by the parallel transport equation $\nabla_{\dot{\gamma}} Y = 0$, where $Y = X / \sqrt{g(X, X)}$.
\begin{algorithm}[ht]
  \small
  \caption{Staying the course}
  \begin{algorithmic}[1]
    \Require Vector field $X \in \X(M)$, initial point $\gamma_0 \in M$, step length $\tau$, tolerance $\rho$, preferred direction $Z_0 \in T_{\gamma_0} M$, samples per iteration $K \in \mathbb{N}$.
    \Ensure Equilibrium $x_\star$ of $X$.
    \State $n$ $\leftarrow$ $0$.
    \Loop
    \State Sample $K$ points from a neighborhood $U \subset M$ of $\gamma_n$.
    \State Obtain diffeomorphisms $\phi \colon U \to \R^m$ and $\psi \colon \R^m \to U$ via manifold learning.
    \While{$\gamma_n \in U$}
    \If{$n > 0$ and $\sqrt{g(X, X)} < \rho$ at $\gamma_n$}
    \State \textbf{return} $x_\star$ $\leftarrow$ $\gamma_n$
    \EndIf
    \State $Z$ $\leftarrow$ vector in $\ker A(Y)$ at $\gamma_n$ \Comment{calculated via partial SVD}
    \State $Z_{n+1}$ $\leftarrow$ $\sign(g(Z_n, Z)) \, \frac{Z}{\sqrt{g(Z, Z)}}$
    \Comment{normalize orientation and magnitude}
    \State $\gamma_{n+1}^k$ $\leftarrow$ $\gamma^k_n + \tau \, Z^k_{n+1} - \tau^2 \sum_{i, j = 1}^m  Z_{n+1}^i \, \Gamma_{ij}^k(\gamma_n) \, Z_{n+1}^j$ for $k = 1, \dotsc, m$
    \State $n$ $\leftarrow$ $n + 1$
    \EndWhile
    \EndLoop
  \end{algorithmic}
  \label{alg:staying-the-course}
\end{algorithm}

\Cref{alg:staying-the-course} proceeds, at a high level, by following these steps:
\begin{enumerate}
\item We draw samples from a neighborhood $U \subset M$ of an initial point $\gamma_0$.
\item We use manifold learning techniques to construct a chart for $U$ with coordinates $\phi = (\phi^1, \dotsc, \phi^m)$.
  Since $M$ is an $m$-dimensional submanifold of $\R^n$, the image $\phi(U)$ lies in $\R^m$ ($m \ll n$ by assumption).
\item Within the chart, we find a curve $\gamma$ that parallel-transports the direction $Y$ of the underlying vector field $X$ at $\gamma_0$.
\item Eventually, the curve $\gamma$ reaches the boundary of the chart.
  At that stage, we replace the point $\gamma_0$ by the latest point in $\gamma$.
\item We repeat the procedure until we reach an equilibrium $x_\star$ of $X$.
\end{enumerate}
In what follows, we detail the different aspects of the algorithm outlined above.

\subsection{Sampling on manifolds}

We assume that we are capable of sampling the vicinity of arbitrary points of $M$.
The choice of a measure on $M$ from which to sample is not important.
However, for the sake of concreteness, let us assume that it is the uniform measure, which is proportional to the Riemannian volume form.

In the numerical examples discussed in this section, we use a Metropolis Markov chain sampler~\cite{metropolis1953,liu2008} but any black-box sampler, such as Hamiltonian Monte Carlo \cite{betancourt2017}, would do.

For systems with holonomic constraints, the sampling mechanism can be as simple as setting a constant $r > 0$ and drawing velocity vectors
\begin{equation*}
  \{ W^{(i)} \in T_x M \mid g(W^{(i)}, W^{(i)})^{1/2} < r, \quad i = 1, \dotsc, K \}
\end{equation*}
uniformly at random to obtain the points
\begin{equation*}
  D = \{ x^{(i)} \in M \mid x^{(i)} = \exp_x W^{(i)}, \quad i = 1, \dotsc, K \}
\end{equation*}
resulting from propagating the initial condition $x \in M$ in the direction of each $W^{(i)}$ using the exponential map~\cite{docarmo1992}.
Similarly, for systems with an inertial manifold, one can propagate an ensemble of trajectories for a short time with initial conditions resulting from perturbing the initial point.

A simple scheme for modifying the vector field $X$ to sample points in the vicinity of a given point is presented at the end of \Cref{sec:transition}.

\subsection{Manifold learning and diffeomorphisms}

Assume that $M$ is a submanifold of $\R^n$ where $n > m$.
Once we have a set $D$ of $K$ points drawn from a neighborhood $U \subset M \subset \R^n$ of the initial point $\gamma_0$, we use manifold learning (see~\cref{sec:manifold-learning} for an introduction to the topic) to reduce the dimensionality of $D$ and, specifically, to get a system of coordinates $\phi \colon U \subset \R^n \to \R^m$ of $M$ in the neighborhood $U$ and the corresponding parameterization $\psi \colon \R^m \to U \subset \R^n$ such that $\psi = \phi^{-1}$ (see \Cref{fig:chart}).

In particular, to obtain the mapping $\phi$, we first use diffusion maps~\cite{coifman2006} with density normalization, taking as our bandwidth parameter the median distance between distinct points of $D$, and we approximate the relevant diffusion maps using Gaussian process regression  (GPR)~\cite{williams1996,rasmussen2006}.

We use GPR for out-of-sample extension of the mapping to diffusion map coordinates of $U \subset M$ and its inverse.
Other options exist, besides GPR.
For instance, Geometric Harmonics~\cite{coifman2006a}, neural networks, or other out-of-sample extension methods~\cite{bengio2003} could also be used to approximate the coordinate mapping.

Thus, $\phi \colon U \to \R^m$ is obtained by fitting the points of the data set $D$ to their corresponding diffusion coordinates.
Likewise, we find $\psi = \phi^{-1}$ by fitting the diffusion coordinates to the original points of the data set $D$.
We present a different approach for obtaining $\psi$ in \Cref{sec:transition}.

\subsection{Pushing the vector field forward}

At each point of the data set $D$ we can sample the vector field $X$ and compute the push-forward of $X$ by the system of coordinates $\phi$.
In other words, we apply the map $X(x) \mapsto D \phi(x) X(x)$, determined by the Jacobian matrix $D \phi$ of the mapping $\phi = (\phi^1, \dotsc, \phi^m)$.

The Jacobian and Hessian matrices of $\phi$ can be calculated using either automatic differentiation~\cite{griewank2008} (AD) or explicit formulas leveraging the fact that a Gaussian process regressor can be written as a linear combination of radial basis functions~\cite[Representer Theorem]{rasmussen2006,hofmann2008}.
Indeed, the Jacobian and Hessian of
\begin{equation*}
  \phi^i(x)
  =
  \sum_{\ell = 1}^K \alpha_\ell \, \e^{-\frac{\| x - x^{(\ell)} \|^2}{2 \nu^2}},
\end{equation*}
are, respectively,
\begin{equation*}
  D \phi^i(x)
  =
  -\frac{1}{\nu^2} \sum_{\ell = 1}^K \alpha_\ell \, \e^{-\frac{\| x - x^{(\ell)} \|^2}{2 \nu^2}} \, (x - x^{(\ell)})
\end{equation*}
and
\begin{equation*}
  D^2 \phi^i(x)
  =
  \frac{1}{\nu^4} \sum_{\ell = 1}^K \alpha_\ell \, \e^{-\frac{\| x - x^{(\ell)} \|^2}{2 \nu^2}} ((x - x^{(\ell)}) (x - x^{(\ell)})^\top - \nu^2 I),
\end{equation*}
where $I$ is the $n \times n$ identity matrix.
Higher-order derivatives can be automated with AD or calculated using closed-form formulas as above.

Alternatively, one may draw a set of samples $D = \{ x^{(i)} \in U \mid i = 1, \dotsc, \lfloor K/2 \rfloor \}$ and propagate them using an explicit Euler scheme
\begin{equation*}
  (x^{(i)})^\prime = x^{(i)} + \Delta t \, X(x^{(i)})
\end{equation*}
to get $D^\prime = \{ (x^{(i)})^\prime \in U \mid i = 1, \dotsc, \lfloor K / 2 \rfloor \}$ for small enough $\Delta t > 0$.
After computing the system of coordinates $\phi$ from data, these two sets of points give us an estimate of the pushforward of $X$.
Indeed,
\begin{equation*}
  \phi_\star X(x^{(i)})
  =
  \left.\frac{\d}{\d t}\right\vert_{t = 0} \phi(\alpha(t))
  \approx
  \frac{\phi((x^{(i)})^\prime) - \phi(x^{(i)})}{\Delta t},
\end{equation*}
where $\alpha \colon [0, 1] \to M$ is a smooth path such that $\alpha(0) = x^{(i)}$ and $\dot{\alpha}(0) = X(x^{(i)})$, for $i = 1, \dotsc, \lfloor K / 2 \rfloor$.

For the sake of simplicity, we abuse notation and denote both the vector field $X$ and its push-forward $\phi_\star X$ by $X$.

In the examples calculated in this paper and in the code available online, we used the Gaussian process regression approach.

\subsection{Pulling the ambient metric back}

Having the mappings $\phi, \psi$ and their respective Jacobian matrices $D \phi, D \psi$ allows us to compute the first fundamental form~\cite{docarmo1992} (metric tensor) of $M$.
The first fundamental form is precisely the pullback of the Euclidean metric in $\R^n$ by $\psi = (\psi^1, \dotsc, \psi^n) \colon \R^m \to U \subset M$,
\begin{equation*}
  g
  =
  \psi^\star \left( \sum_{\ell = 1}^n \d x^\ell \otimes \d x^\ell \right)
  =
  \sum_{i, j = 1}^m g_{ij} \d \phi^i \otimes \d \phi^j,
\end{equation*}
where the components $g_{ij}$ are the entries of the matrix $D \psi^\top D \psi$.
We will see later that the metric tensor $g$ allows us to compute the Christoffel symbols as well as to establish a convergence criterion for our algorithm.

The method above is what we have used for the examples in this paper but it is not the only possible avenue for obtaining $g$.
Another source of estimators for the components $g_{ij}$ can be obtained from Green's identity~\cite{chavel1984}
\begin{equation*}
  \Div \grad( f h )
  =
  h \Delta f + f \Delta h + 2 g(\grad f, \grad h)
\end{equation*}
for $f, h \in C^\infty(M)$.
Yet another source of estimators stems from the short time asymptotics of the heat kernel~\cite{varadhan1967}.
We refer the reader to~\cite{perrault-joncas2013,berry2020} for other methods of computing $g$.

\subsection{The parallel transport equation}

In order to obtain the direction of the tangent vector to the curve $\gamma$ at the current point, we must solve the parallel transport equation
\begin{equation}
  \label{eq:parallel-transport}
  \nabla_Z Y
  =
  \sum_{k = 1}^m \left\{
    \sum_{i = 1}^m
    \left(
      \frac{\partial Y^k}{\partial \phi^i}
      +
      \sum_{j = 1}^m \Gamma_{ij}^k Y^j
    \right)
    Z^i
  \right\}
  \partial_k
  =
  0
\end{equation}
where the vector $Z \in T_{\gamma(s)} M$ is the unknown.

The Christoffel symbols $\Gamma_{jk}^\ell$ for $j, k, \ell = 1, \dotsc, m$ are readily obtained from our characterization of the Riemannian metric $g$ by the formula
\begin{equation*}
  \Gamma_{jk}^\ell
  =
  \sum_{i = 1}^m g^{\ell i} \left( \frac{\partial  g_{ij}}{\partial \phi^k} + \frac{\partial  g_{ik}}{\partial \phi^j} - \frac{\partial g_{jk}}{\partial \phi^i} \right),
\end{equation*}
where $g^{ij}$ denotes the coefficients of the inverse metric (\emph{i.e.,} the components of the inverse matrix of $(g_{ij})_{i, j = 1}^m$).

Solving~\eqref{eq:parallel-transport} for $Z$ is equivalent to finding the one-dimensional nullspace (see \Cref{thm:riemannian-field-of-lines}) of the $m \times m$ matrix $A(Y)$, as defined in equation~\eqref{eq:nabla}.
The step in line 9 of \Cref{alg:staying-the-course} consists of numerically obtaining $\ker A(Y)$.
This can be accomplished by calculating the singular value decomposition (SVD)~\cite{golub2013} using the shift\--and\--invert method within the underlying symmetric eigensolver \cite[Chapter 6]{zhaojun2000} to efficiently obtain the smallest singular value and its associated right singular vector without computing the full SVD.

Again, by \Cref{thm:riemannian-field-of-lines}, the sought-after vector $Z$ spans a one-dimensional subspace of $T_{\gamma(s)} M$ but its magnitude and orientation must be normalized before we can use it to solve a differential equation for $\gamma$.
Indeed, solving the initial value problem
\begin{equation}
  \label{eq:ode-gamma}
  \left\{
    \begin{aligned}
      &\dot{\gamma}(s) = Z(\gamma(s)), \\
      &\gamma(0) = \gamma_0,
    \end{aligned}
  \right.
\end{equation}
as is, would not work in general because $Z$ is not guaranteed to be smooth along the curve $\gamma(s)$ due to the fact that its orientation and magnitude can change abruptly from one point to another.
In order to use $Z$ in practice, we normalize it dividing by $\sqrt{g(Z, Z)}$ so that it has unit length and then we ensure that it is oriented consistently with respect to $\dot{\gamma}$ at the previous time step.
We do so by possibly introducing a change of sign to enforce that the angle between $\dot{\gamma}$ and $Z$ be less than $\pm \frac{\pi}{2}$ radians (this is accomplished by line 10 of \Cref{alg:staying-the-course}).

Once we have suitably normalized $Z$, we can solve~\eqref{eq:ode-gamma} and find another point $\gamma_{n+1}$ along the curve $\gamma$.
To derive line 11 of \Cref{alg:staying-the-course}, we begin by writing the second order ODE $\ddot{\gamma}^k = -\sum_{i,j = 1}^m \dot{\gamma}^i \Gamma_{ij}^k(\gamma) \dot{\gamma}^j$ as the first order system
\begin{equation*}
  \left\{
    \begin{aligned}
      &\dot{\gamma}^k = W^k, \\
      &\dot{W}^k = -\sum_{i, j = 1}^m W^i \Gamma_{ij}^k W^j
    \end{aligned}
  \right.
\end{equation*}
with initial conditions $\gamma^k(0) = \gamma_n^k$ and $W^k(0) = Z_{n+1}^k$.
Next, we apply the explicit Euler scheme for one time step $\tau$ to obtain
\begin{equation*}
  \gamma^k(\tau) = \gamma_n^k + \tau Z_{n+1}^k - \tau^2 \sum_{i, j = 1}^m Z_{n+1}^i \Gamma_{ij}^k(\gamma_n) Z_{n+1}^j + o(\tau^2)
\end{equation*}
as $\tau \to 0$ for $k = 1, \dotsc, m$.
This integrator could be replaced by a better suited numerical algorithm (e.g., a symplectic scheme, etc.)~\cite{hairer2006}.

\subsection{Transition to a new chart}
\label{sec:transition}

At some instant $T > 0$, the solution $\gamma$ of~\eqref{eq:ode-gamma} reaches a point $\gamma(T)$ where the coordinate mapping $\phi$ is no longer valid.
Since $\phi$ is determined by a point-cloud, this typically occurs when $\gamma(T)$ is close to the boundary of the cloud and our estimate of the mapping $\phi$ becomes unreliable due to sparse sampling (high variance) within that region.

Gaussian process regression (GPR), which we use to estimate $\phi$, provides us with a tool to quantitatively decide whether we are close to the boundary.
If $\Sigma$ is the covariance matrix of the Gaussian process $\phi$ at a point $\gamma(s)$, then we can decide to abandon the current coordinate patch based on the criterion $\| \Sigma \| > \eta$, where we take $\| \cdot \|$ to be the operator norm or the Frobenius norm, and $\eta > 0$ is some prescribed threshold value.

Once we decide to discard the coordinate patch, we map the latest point $\gamma_n$ to the ambient space using $\psi = \phi^{-1}$ and we also map the latest vector $Z_n$ to $T_{\gamma_n} M$ using the push-forward $\psi_\star$.
After that, we are ready to construct a new coordinate patch repeating the steps of our algorithm until convergence to an equilibrium.

Instead of using a Gaussian process for $\psi$, we could alternatively solve the stochastic differential equation
\begin{equation}
  \label{eq:umbrella-sampling}
  \d w
  =
  \left(
    X(w) - \zeta \, D \phi(w)^\top (\phi(w) - \gamma_n)
  \right) \d t
  + \sqrt{2 \beta^{-1}} \d B_t, \\
\end{equation}
where $w \in \R^n$, $\zeta > 0$ and $\beta > 0$ are constants, and $B_t$ is a standard $n$-dimensional Brownian motion.
This type of biased sampling also arises as a step in umbrella sampling~\cite{torrie1974} (in applications to molecular dynamics, one typically resorts to packages such as Colvars~\cite{fiorin2013}).
Then the average point $\langle w \rangle$ and the average vector $\langle X(w) \rangle$ of the solution of~\eqref{eq:umbrella-sampling} yield
\begin{equation*}
  \psi(\gamma_n) = \langle w \rangle
  \quad \text{and} \quad
  \psi_\star Z_n = \langle X(w) \rangle.
\end{equation*}
This approach of sampling by simulating a stochastic differential equation is more amenable to the case of a high-dimensional ambient space.

\subsection{Convergence to an equilibrium}

In lines 6--8 of \Cref{alg:staying-the-course} we decide whether the latest point $\gamma_n$ is sufficiently close to an equilibrium of $X$ by measuring the magnitude of $X$ at $\gamma_n$.
We use the metric $g$ to determine whether we have reached the vicinity of an equilibrium of $X$.
If $g(X, X) = 0$ at a given point, this implies that $X = 0$ at that point.
Thus, for a fixed threshold $\rho > 0$, we take the condition $g(X, X) < \rho$ as a convergence criterion for our algorithm.

\subsection{Example}

We revisit here the case of the M\"uller-Brown potential on the sphere (see \Cref{ex:sphere}) but this time we forego the knowledge of the manifold and we construct its charts on the fly via sampling and dimensionality reduction, as explained earlier in this section.
Figures~\ref{fig:muller-sphere-dmaps-000}, \ref{fig:muller-sphere-dmaps-050}, and~\ref{fig:muller-sphere-dmaps-101} illustrate the behavior of our algorithm from the initial iteration up until a nearby saddle point is found.
The main difference between this example and the ones presented in \Cref{sec:isoclines} is that we use a Metropolis sampler to construct the point-clouds, diffusion maps to obtain the local charts, and we fit Gaussian processes to map back and forth between $\R^n$ and the charts, as described earlier in this section.
\begin{figure}[ht]
  \begin{subfigure}[b]{0.4\textwidth}
    \includegraphics[width=\columnwidth,trim=13cm 5cm 5.5cm 5cm,clip]{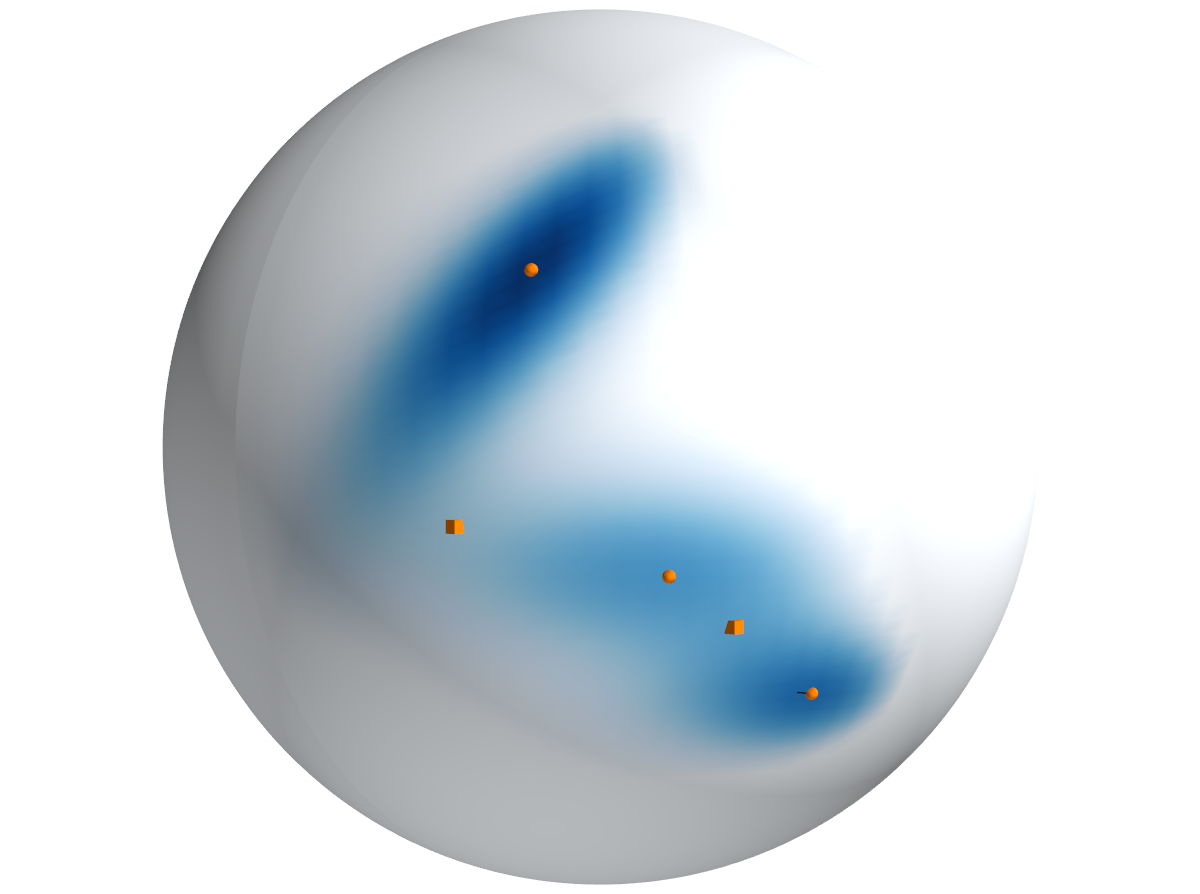}
    \caption{Sphere $\mathbb{S}^2$.}
  \end{subfigure}
  \hfill
  \begin{subfigure}[b]{0.4\textwidth}
    \includegraphics[width=\columnwidth,trim=0.75cm 0.76cm 0.75cm 0.75cm,clip]{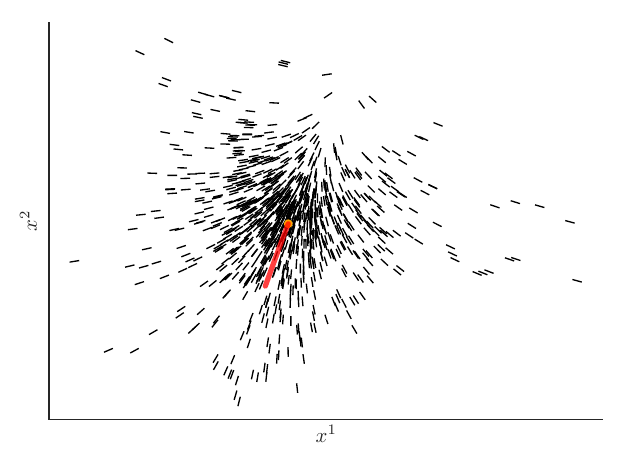}
    \caption{Local chart.}
  \end{subfigure}
  \caption{M\"uller potential on $\mathbb{S}^2$ (diffusion maps). Iteration \#1.}
  \label{fig:muller-sphere-dmaps-000}
\end{figure}

\begin{figure}[ht]
  \begin{subfigure}[b]{0.4\textwidth}
    \includegraphics[width=\columnwidth,trim=13cm 5cm 5.5cm 5cm,clip]{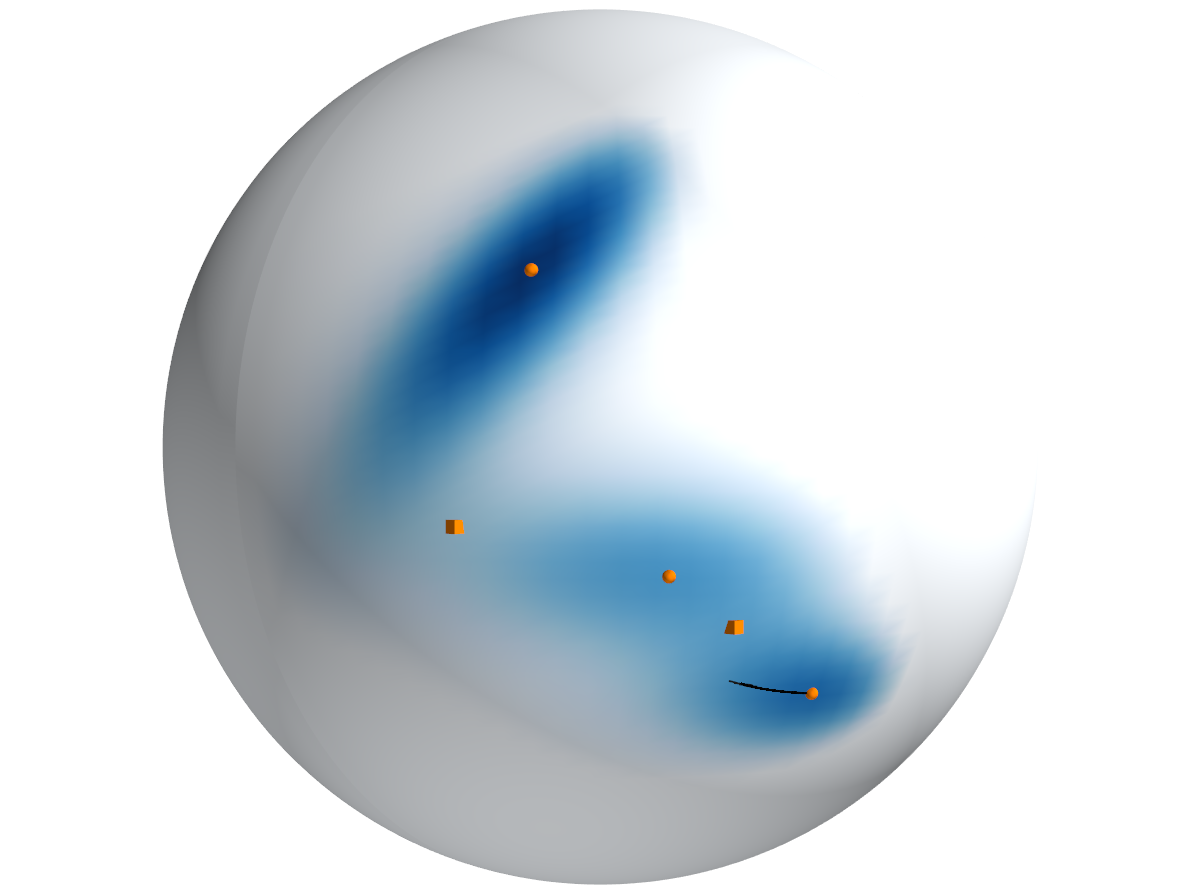}
    \caption{Sphere $\mathbb{S}^2$.}
  \end{subfigure}
  \hfill
  \begin{subfigure}[b]{0.4\textwidth}
    \includegraphics[width=\columnwidth,trim=0.75cm 0.76cm 0.75cm 0.75cm,clip]{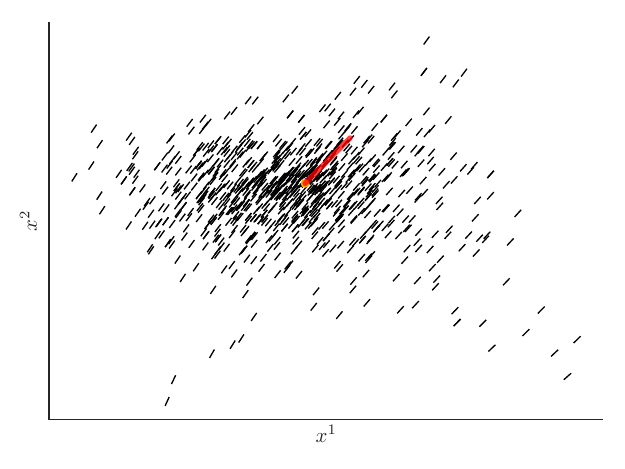}
    \caption{Local chart.}
  \end{subfigure}
  \caption{M\"uller potential on $\mathbb{S}^2$ (diffusion maps). Iteration \#51.}
  \label{fig:muller-sphere-dmaps-050}
\end{figure}

\begin{figure}[ht]
  \begin{subfigure}[b]{0.4\textwidth}
    \includegraphics[width=\columnwidth,trim=13cm 5cm 5.5cm 5cm,clip]{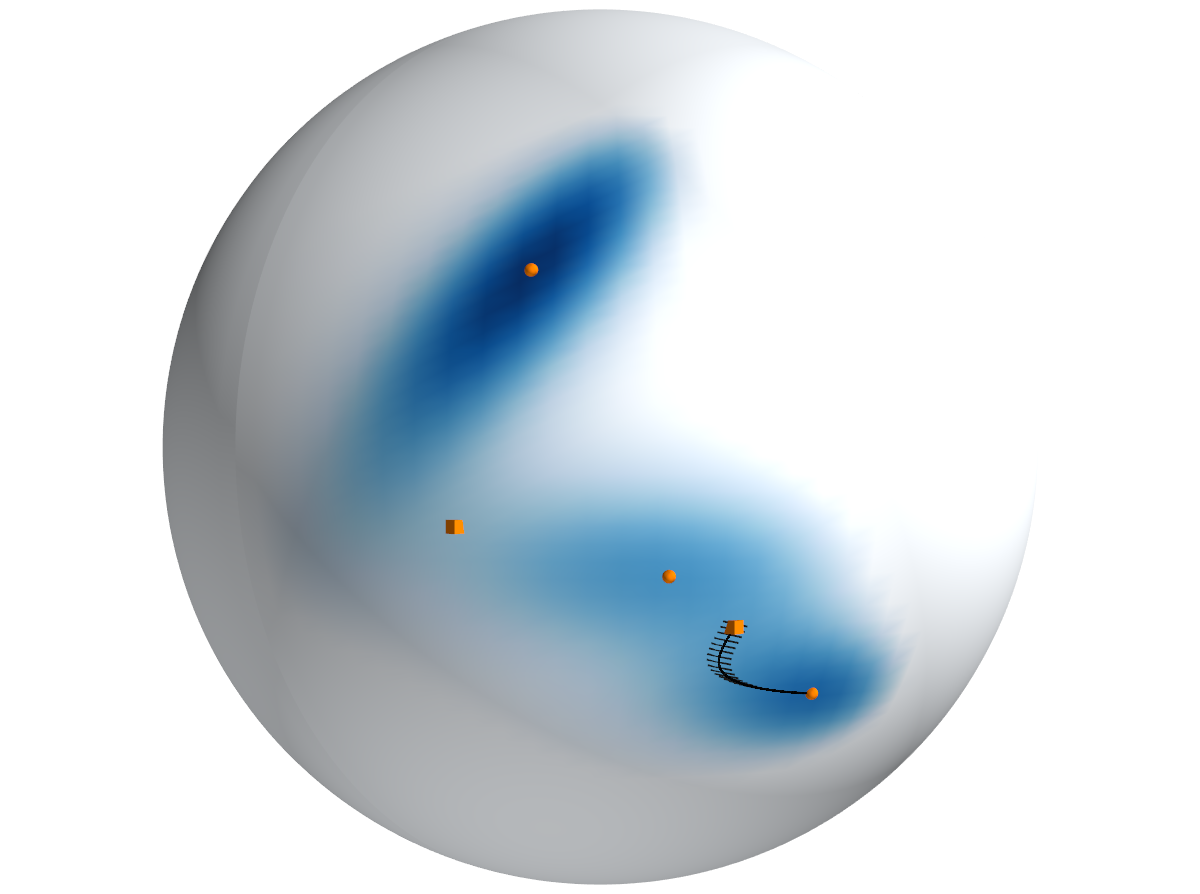}
    \caption{Sphere $\mathbb{S}^2$.}
  \end{subfigure}
  \hfill
  \begin{subfigure}[b]{0.4\textwidth}
    \includegraphics[width=\columnwidth,trim=0.75cm 0.76cm 0.75cm 0.75cm,clip]{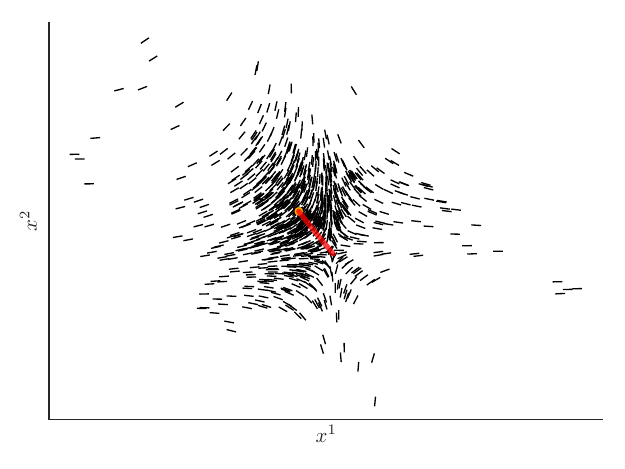}
    \caption{Local chart.}
  \end{subfigure}
  \caption{M\"uller potential on $\mathbb{S}^2$ (diffusion maps). Iteration \#102.
  }
  \label{fig:muller-sphere-dmaps-101}
\end{figure}

\subsection{Sporadic switch to steepest descent}

In a gradient system, where $X = -D E$ is the force associated to the potential energy function $E \in C^\infty(M)$, the trajectory $\gamma$ computed by \Cref{alg:staying-the-course} might take the system into high-energy regions, which could lead to numerical instability.
Under those circumstances, switching to steepest descent and converging to the local minimum in the current basin of attraction, still provides a useful way of exploring the energy landscape of the system.

\subsection{Limitations of the method}

Our method hinges on the success of its underlying building blocks at each step for it to converge to a fixed point.
Here, we enumerate the failure modes of each building block:
\begin{description}
\item [Manifold learning] The success of the chosen manifold learning algorithm is crucial: if it fails, then, of course, so does our algorithm.
  Appendix~\ref{sec:manifold-learning} details our use of the diffusion maps method.
\item [Sampling on manifolds] The curvature of the underlying manifold increases the difficulty of sampling the neighborhood of a given point.
\item [Estimation of the vector field] If the underlying vector field changes rapidly in space, then it will be harder to construct accurate estimates of it.
\item [Out of distribution (OOD) generalization] In our particular implementation, we rely on Gaussian process regression to be able to sample OOD points.
  The greater the diameter of the set of samples drawn in the vicinity of the given point, the better the GP regression will perform.
\item [Numerical integration] The absence of a corrector step (as in numerical continuation schemes) implies that the integration scheme must be accurate enough to remain close to the true isocline while it advances toward an equilibrium point.
\end{description}


%% file: conclusion.tex
\section{Concluding remarks}
\label{sec:conclusion}

We have introduced an algorithm for locating equilibria of vector fields on compact Riemannian manifolds that are defined by (iteratively sampled) point-clouds.
Our method works by finding a path along which the direction of the vector field remains parallel.
We have shown that the parallelism condition guarantees that such a path will end at an equilibrium.
When the manifolds are defined by point-clouds (instead of by the roots of smooth functions or by an atlas), our proposed scheme relies on sampling, manifold learning techniques (namely, diffusion maps), and Gaussian process regression to carry out the relevant differential-geometric computations and obtain the sought-after path that joins equilibria.
We have presented examples of our algorithm on manifolds with both constant positive and negative curvature on a model dynamical system commonly used in computational chemistry and in a system for which closed-form formulas exist (see \Cref{sec:simple-example-in-closed-form}).

We believe our approach is as parsimonious as possible: we do not know the manifold {\em a priori}, nor do we assume known collective variables; our search proceeds along a curve.
This should be contrasted to methods such as metadynamics, that sample the entire space of collective variables, and our own iMapD, which systematically explores entire basins of attraction in the space of collective variables.
Furthermore, a single initial point suffices to start our algorithm (as opposed to the string method and its variants, that need two endpoints).

A proof of concept code for our algorithm can be found at \url{https://github.com/jmbr/staying-the-course-code/}.


%% file: acknowledgments.tex
\section{Acknowledgments}

We are grateful to Andrew Duncan (Imperial College London, UK) for pointing us to the literature on reduced gradient following and to Tyrus Berry (George Mason University, USA) for valuable feedback on an early version of this manuscript.


%% file: manifold-learning.tex
\section{A whirlwind tour of kernel-based manifold learning}
\label{sec:manifold-learning}

Let $M$ be a closed (\emph{i.e.,} compact and boundaryless) Riemannian manifold.
The heat kernel (\emph{i.e.,} the fundamental solution of the heat equation) of $M$ can be written as the spectral decomposition
\begin{equation}
  \label{eq:spectral-decomposition}
  k_\epsilon(x, y)
  =
  \sum_{n = 0}^\infty \e^{-\lambda_n \epsilon} \phi_n(x) \phi_n(y)
\end{equation}
where $x, y \in M$, $\epsilon \ge 0$, and $\lambda_n \in [0, \infty)$ are eigenvalues of the Laplace-Beltrami operator on $M$ with $\phi_n \in C^\infty(M)$ being the corresponding eigenfunctions for $n = 0, 1, \dotsc$
The eigenvalues are in non-decreasing order, $0 = \lambda_0 < \lambda_1 \le \lambda_2 \le \dotsb$.

The heat kernel allows us~\cite{berard1994} to embed $M$ in $\ell^2$ by the projection of a map $x \mapsto k_\epsilon(x, \cdot)$ onto the spectral basis of $L^2(M)$ given by $\phi_n$.
Indeed, given $x \in M$, we can regard~\eqref{eq:spectral-decomposition} as a Fourier expansion whose sequence of coefficients $(\e^{-\lambda_n \epsilon} \phi_n(x))_{n = 0}^\infty \in \ell^2$ yield the desired embedding.
If a spectral gap $\lambda_m \ll \lambda_{m+1}$ exists, then it justifies approximating the aforementioned embedding into $\ell^2$ by an embedding into the finite-dimensional vector space $\R^m$.

The diffusion maps algorithm~\cite{coifman2006} used in this paper is a practical method for computing the embedding we have discussed when the manifold is given by a point-cloud.
Note that we could have replaced diffusion maps by other manifold learning techniques such as kernel principal component analysis (KPCA)~\cite{scholkopf1997}, Iso\-map~\cite{tenenbaum2000}, Locally Linear Embedding~\cite{roweis2000,donoho2003}, Local Tangent Space Alignment~\cite{zhang2004}, Autoencoders~\cite{rumelhart1986,kingma2013,miolane2020}, etc. (some of the aforementioned manifold learning methods can be regarded as special cases of KPCA \cite{ham2004}).

Figure~\ref{fig:swiss-roll} shows an example of the application of the diffusion maps algorithm to a two-dimensional manifold, the Swiss roll, defined by a point-cloud in three-dimensional Euclidean space.
In this case, two eigenfunctions $(\phi_1, \phi_2)$ suffice for parameterizing the Swiss roll.
\begin{figure}[ht]
  \centering
    \begin{subfigure}[t]{0.495\columnwidth}
      {\includegraphics[width=0.9\columnwidth,trim=1cm 1cm 1cm 1cm,clip]{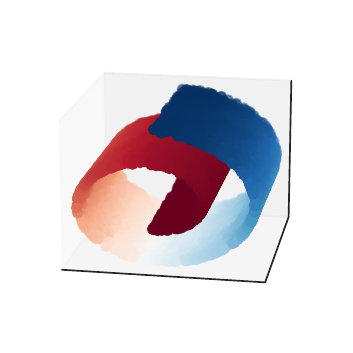}}
      \caption{Data set colored by the value of $\phi_1$.}
    \end{subfigure}
    \hfill
    \begin{subfigure}[t]{0.495\columnwidth}
      {\includegraphics[width=0.9\columnwidth,trim=1cm 1cm 1cm 1cm,clip]{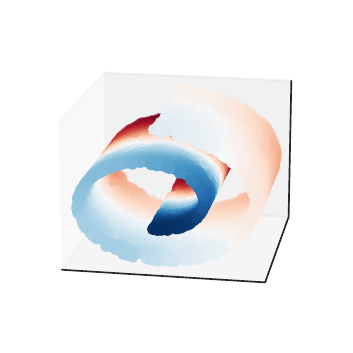}}
      \caption{Data set colored by the value of $\phi_2$.}
    \end{subfigure}
    \caption{Dimensionality reduction of an example data set known as the Swiss roll.
      The colors in each figure correspond to the first and second non-trivial eigenfunctions of the Laplace-Beltrami operator of the manifold on which the point-cloud was sampled.
    The two eigenfunctions together provide a (two-dimensional) system of coordinates for the Swiss roll manifold.}
  \label{fig:swiss-roll}
\end{figure}


%% file: closed-form-example.tex
\section{Simple example in closed form}
\label{sec:simple-example-in-closed-form}

In this section we compute in closed form the field of lines defining the isoclines of the simple vector field corresponding to the potential $E(x^1, x^2, x^3) = x^1 x^2 x^3$ defined on the sphere (see \Cref{fig:simple-potential}).
\begin{figure}[ht]
  \centering
  \includegraphics[width=0.4\columnwidth,trim=0.25cm 0.75cm 0.25cm 0.75cm,clip]{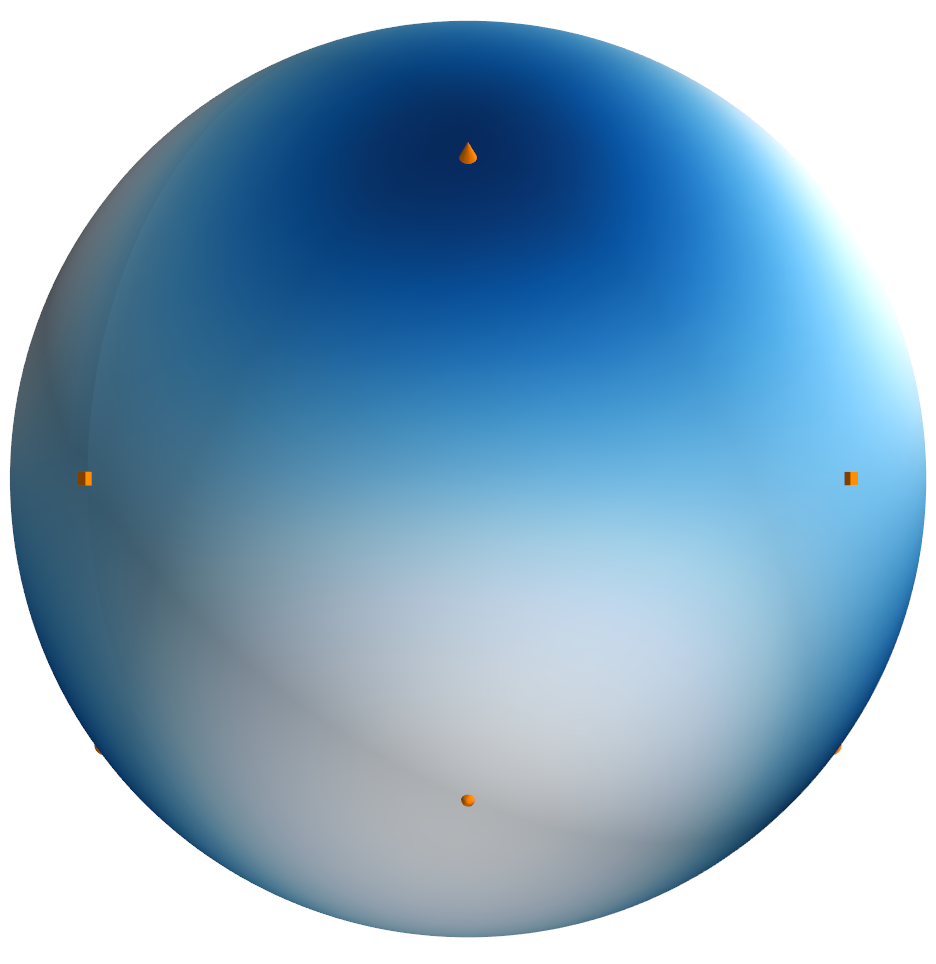}
  \caption{Potential $E(x^1, x^2, x^3) = x^1 x^2 x^3$ on $\S^2$ represented as a heat map.}
  \label{fig:simple-potential}
\end{figure}

\subsection{Stereographic projection}
\label{sec:stereographic-projection}

Consider the 2-sphere
\begin{equation*}
  \S^2
  =
  \{ (x^1, x^2, x^3) \in \R^3 \mid (x^1)^2 + (x^2)^2 + (x^3)^2 = 1 \}.
\end{equation*}
The stereographic projection provides us with the atlas $\{ (U_{+1}, \phi_{+1}), (U_{-1}, \phi_{-1}) \}$, where the systems of coordinates $\phi_{\pm 1} \colon U_{\pm 1} \subset M \to \R^2$ are given by
\begin{equation*}
  \phi_{\pm 1}(x) = (
  \frac{x^1}{1 \mp x^3}, \frac{x^2}{1 \mp x^3}
  ) \in \R^2,
\end{equation*}
Let $p = (p^1, p^2) \in \R^2$ and $q = (q^1, q^2) \in \R^2$.
The corresponding parameterizations are given by
\begin{equation*}
  \phi_{+1}^{-1}(p) = \frac{1}{\nu + 1} (2 p^1, 2 p^2, \nu - 1)
  \quad \text{and} \quad
  \phi_{-1}^{-1}(q) = \frac{1}{\chi + 1} (2 q^1, 2 q^2, -\chi + 1),
\end{equation*}
where $\nu = (p^1)^2 + (p^2)^2$ and $\chi = (q^1)^2 + (q^2)^2$.

Let $U_{\pm 1}^\prime = \phi_{\pm 1}(U_{\pm 1})$.
The transition maps between charts, $(\phi_{-1} \circ \phi_{+1}^{-1}) \colon U^\prime_{+1} \cap U^\prime_{-1} \to U^\prime_{-1}$ and $(\phi_{+1} \circ \phi_{-1}^{-1}) \colon U^\prime_{+1} \cap U^\prime_{-1} \to U^\prime_{+1}$, are defined by:
\begin{equation*}
  (\phi_{-1} \circ \phi_{+1}^{-1}) (p) = \frac{p}{\nu}
  \qquad \text{and} \qquad
  (\phi_{+1} \circ \phi_{-1}^{-1}) (q) = \frac{q}{\chi}.
\end{equation*}

For the sake of simplicity, we focus on the chart $U = U_{+1}$ with coordinates $\phi = \phi_{+1}$ and parameterization $\psi = \phi_{+1}^{-1}$.
In the chart $U$, the potential energy is given by
\begin{equation*}
  E(p^1, p^2)
  =
  4 p^1 p^2 \frac{\nu - 1}{\nu + 1}.
\end{equation*}

\subsection{Metric tensor and Christoffel symbols}

The metric tensor $g$ in the stereographic projection on the chart $U$ is given by
\begin{equation*}
  g
  =
  \frac{4}{\nu + 1} \left(
    \d p^1 \otimes \d p^1 + \d p^2 \otimes \d p^2
  \right).
\end{equation*}
The non-redundant Christoffel symbols are
\begin{equation*}
  \Gamma_{11}^1 =
  \Gamma_{12}^2 =
  -\Gamma_{22}^1 = -\frac{2 p^1}{\nu + 1}
  \quad \text{and} \quad
  -\Gamma_{11}^2 =
  \Gamma_{12}^1 =
  \Gamma_{22}^2 = -\frac{2 p^2}{\nu + 1}.
\end{equation*}

\subsection{A vector field and its associated line field}

The vector field corresponding to the potential $E$ in the chart $U$ is given by $X = X^1 \frac{\partial}{\partial p^1} + X^2 \frac{\partial}{\partial p^2}$, where
\begin{align*}
  X^1
  =
  \frac{4 p^2}{(\nu + 1)^{4}} \,
  \left( 3 (p^1)^4+2 (p^1)^2 (p^2)^2-(p^2)^4-8 (p^1)^2 + 1 \right)
\end{align*}
and
\begin{equation*}
  X^2
  =
  -\frac{4 p^1}{(\nu + 1)^{4}} \,
  \left((p^1)^4-2 (p^1)^2 (p^2)^2-3 (p^2)^4+8 (p^2)^2 - 1 \right)
\end{equation*}
The norm $\sqrt{g(X, X)}$, given by
\begin{align*}
  \tfrac{1}{8}
  (\nu + 1)^5
  \sqrt{g(X, X)}
  &=
    \left((p^1)^{10}+5 (p^1)^{8}(p^2)^{2}+10 (p^1)^{6}(p^2)^{4}+10 (p^1)^{4}(p^2)^{6} \right.\\
  &+5 (p^1)^{2}(p^2)^{8} + (p^2)^{10} -32 (p^1)^{6}(p^2)^{2}-64 (p^1)^{4}(p^2)^{4} \\
  &-32 (p^1)^{2}(p^2)^{6}-2 (p^1)^{6} +74 (p^1)^{4}(p^2)^{2} +74 (p^1)^{2}(p^2)^{4} \\
  &-2  \left. (p^2)^{6} -32 (p^1)^{2}(p^2)^{2} +(p^1)^{2}+(p^2)^{2} \right)^{1/2},
\end{align*}
allows us to compute the normalized vector field $Y = X / \sqrt{g(X, X)}$.

The matrix $A(Y) \in \R^{2 \times 2}$ has components
\begin{align*}
  a_{11}
  &= \frac{(\nu + 1) ((X^2)^2 \frac{\partial X^1}{\partial p^1} - X^1 X^2 \frac{\partial X^2}{\partial p^1}) - 2 (X^1)^2 X^2 p^2 - 2 (X^2)^3 p^2}{2 \zeta^{3/2}} \\
  a_{12}
  &= \frac{(\nu + 1) ((X^2)^2 \frac{\partial X^1}{\partial p^2} - X^1 X^2 \frac{\partial X^2}{\partial p^2}) + 2 (X^1)^2 X^2 p^1 + 2 (X^2)^3 p^1}{2 \zeta^{3/2}} \\
  a_{21}
  &= \frac{(\nu + 1)( (X^1)^2 \frac{\partial X^2}{\partial p^1} - X^1 X^2 \frac{\partial X^1}{\partial p^1}) + 2 (X^1)^3 p^2 + 2 X^1 (X^2)^2 p^2}{2 \zeta^{3/2}} \\
  a_{22}
  &= \frac{(\nu + 1)((X^1)^2 \frac{\partial X^2}{\partial p^2} - X^1 X^2 \frac{\partial X^1}{\partial p^2}) - 2 (X^1)^3 p^1 - 2 X^1 (X^2)^2 p^1}{2 \zeta^{3/2}},
\end{align*}
where $\zeta = (X^1)^2 + (X^2)^2$.

It is straight-forward to verify that the field of lines associated with $\dot{\gamma}$ in $\nabla_{\dot{\gamma}} Y = 0$ or, equivalently, the nullspace of $A(Y)$, is spanned by the vector
$L = L^1 \frac{\partial}{\partial p^1} + L^2 \frac{\partial}{\partial p^2}$, where the components
\begin{equation*}
  L^1
  =  (\nu + 1) \left( X^1 \frac{\partial X^2}{\partial p^2} - X^2 \frac{\partial X^1}{\partial p^2} \right) - 2 \zeta p^1
\end{equation*}
and
\begin{equation*}
  L^2
  = (\nu + 1) \left( X^2 \frac{\partial X^1}{\partial p^1} - X^1 \frac{\partial X^2}{\partial p^1} \right) - 2 \zeta p^2
\end{equation*}
are polynomials in $p^1$ and $p^2$ of degree 11 (recall that $L$ is defined up to a non-zero scalar factor).
Indeed,
\begin{align*}
  L^1
  &= (p^1)^{11}+5 (p^1)^{9}(p^2)^{2}+10 (p^1)^{7}(p^2)^{4}+10 (p^1)^{5}(p^2)^{6}+5 (p^1)^{3}(p^2)^{8}+p^1 (p^2)^{10} \\
  &-5 (p^1)^{9}+28 (p^1)^{7}(p^2)^{2}+50 (p^1)^{5}(p^2)^{4}-4 (p^1)^{3}(p^2)^{6}-21 p^1 (p^2)^{8}-6 (p^1)^{7} \\
  &-130 (p^1)^{5}(p^2)^{2}-130 (p^1)^{3}(p^2)^{4}-6 p^1 (p^2)^{6}+6 (p^1)^{5}+124 (p^1)^{3}(p^2)^{2} \\
  &+6 p^1 (p^2)^{4} +5 (p^1)^{3}-11 p^1 (p^2)^{2}-p^1
\end{align*}
and
\begin{align*}
  L^2
  &=
    (p^1)^{10}p^2+5\,(p^1)^{8}(p^2)^{3}+10\,(p^1) ^{6}(p^2)^{5}+10\,(p^1)^{4}(p^2)^{7}+5\,(p^1)^{2}(p^2)^{9} \\
  &+(p^2)^{11} -21\,(p^1)^{8}p^2-4\,(p^1) ^{6}(p^2)^{3}+50\,(p^1)^{4}(p^2)^{5}+28\,(p^1)^{2} (p^2)^{7}-5\,(p^2)^{9} \\
  &-6\,(p^1)^{6}p^2 -130\,(p^1)^{4}(p^2)^{3}-130\,(p^1)^{2}(p^2)^{5}-6\,(p^2) ^{7}+6\,(p^1)^{4}p^2 \\
  &+124\,(p^1)^{2}(p^2)^{3}+6\,(p^2)^{5} -11\,(p^1)^{2}p^2+5\,(p^2)^{3}-p^2.
\end{align*}

\begin{figure}[ht]
  \centering
  \begin{subfigure}[t]{0.475\textwidth}
    \includegraphics[width=\columnwidth,trim=0cm 0cm 1cm 1.25cm,clip]{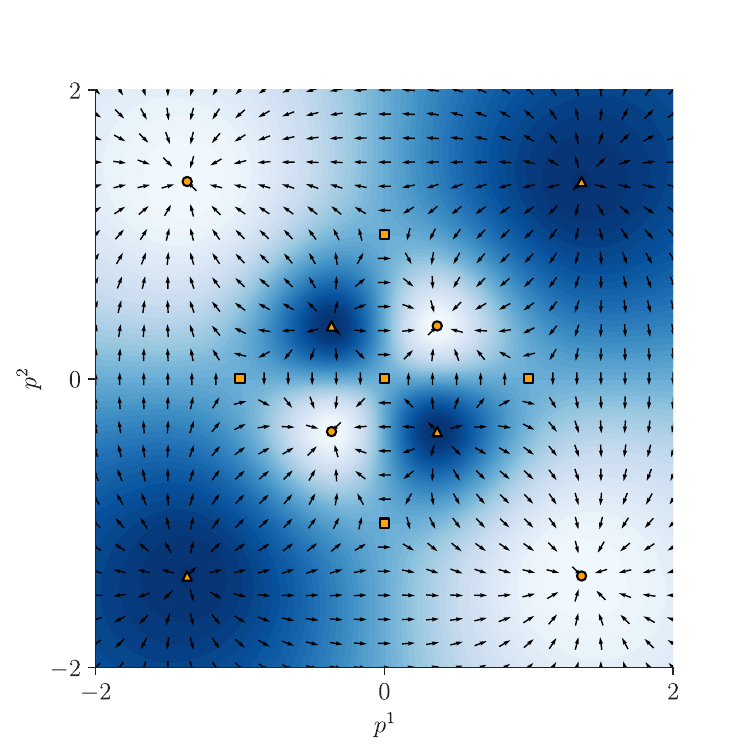}
    \caption{Directions of the steepest descent vector field $X$ in the chart $U$.}
  \end{subfigure}
  \hfill
  \begin{subfigure}[t]{0.475\textwidth}
    \includegraphics[width=\columnwidth,trim=0cm 0cm 1cm 1.25cm,clip]{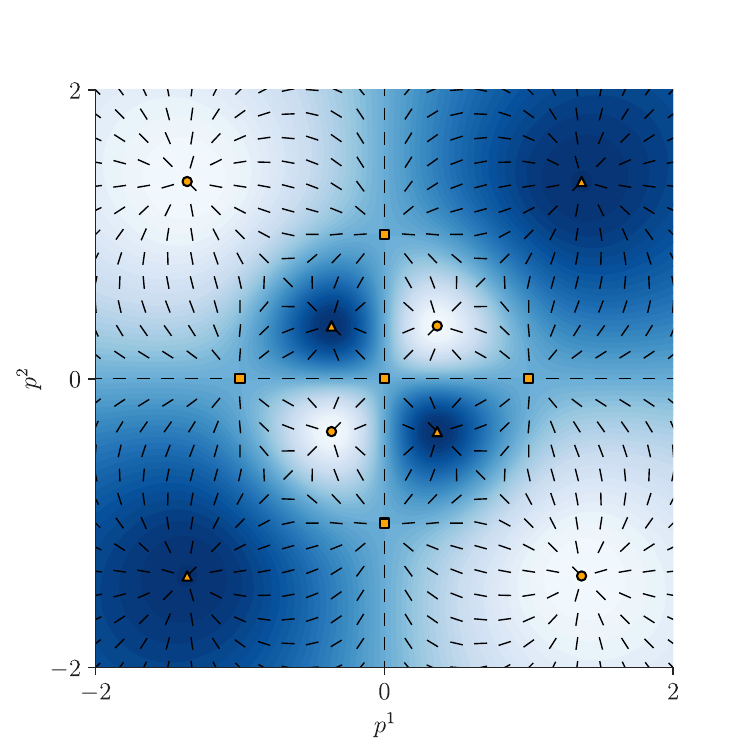}
    \caption{Field of lines determined by the steepest descent vector field $X$ in the chart $U$.
      Generalized isoclines are curves whose tangents are parallel to the field of lines.}
  \end{subfigure}
  \caption{Potential $E(x^1, x^2, x^3) = x^1 x^2 x^3$ on the stereographic projection and its associated vector field and field of lines.
    Sinks are represented by (\textcolor{orange}{$\bullet$}), sources are represented by (\textcolor{orange}{$\blacktriangle$}), and saddles are represented by (\textcolor{orange}{$\blacksquare$}).}
\end{figure}

Once we have characterized the unique field of lines $L$ induced by $X$, we are ready to compute generalized isoclines of $X$.
We end this section by illustrating (see \Cref{fig:closed-isoclines}) the existence of closed isoclines and how the step length in their numerical integration can affect the results.

\begin{figure}[ht]
  \centering
  \begin{subfigure}[t]{0.475\textwidth}
    \includegraphics[width=\columnwidth,trim=0cm 0cm 1cm 1.25cm,clip]{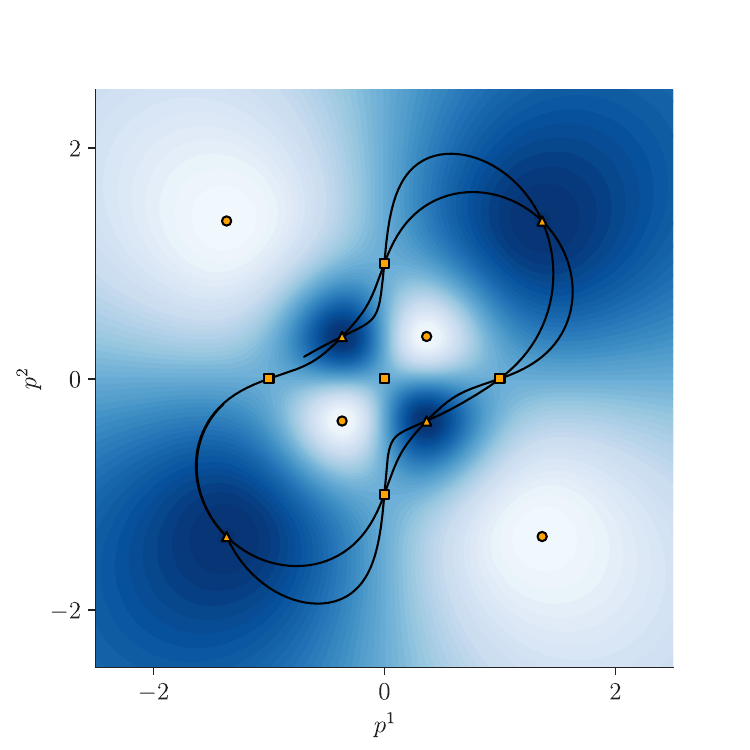}
    \caption{Isocline computed with $\tau = 10^{-3}$ and $2 \times 10^4$ steps.}
  \end{subfigure}
  \hfill
  \begin{subfigure}[t]{0.475\textwidth}
    \includegraphics[width=\columnwidth,trim=0cm 0cm 1cm 1.25cm,clip]{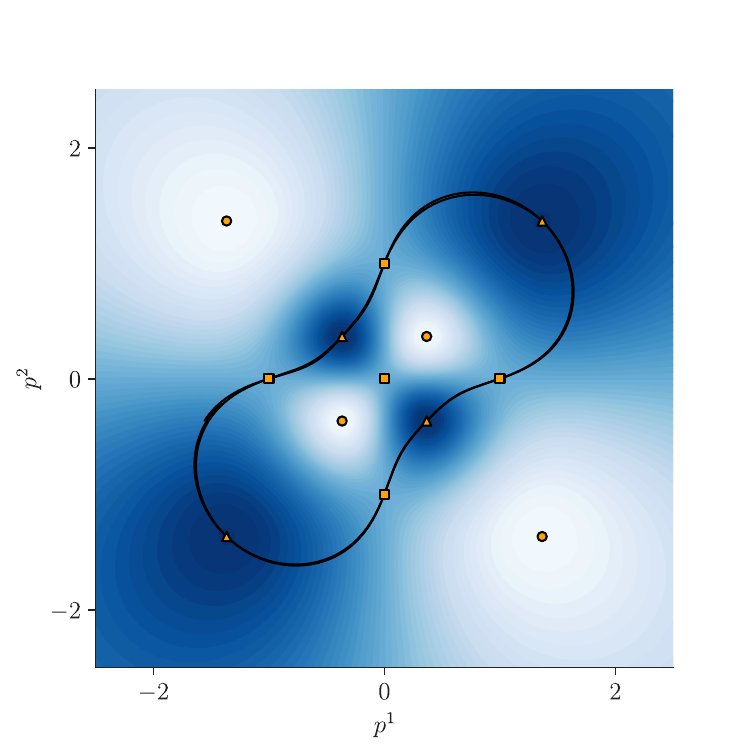}
    \caption{Isocline computed with $\tau = 10^{-4}$ and $2 \times 10^5$ steps.}
  \end{subfigure}
  \caption{Two numerically computed isoclines of the potential $E(x^1, x^2, x^3) = x^1 x^2 x^3$ on $\S^2$, integrated with the same initial conditions and having the same length but differing by the step length $\tau$ used in their numerical calculation.
    Note that the distortion in the coarsely obtained isocline decreases with a finer integration step.}
  \label{fig:closed-isoclines}
\end{figure}


%% file: path-connections.tex
\section{Three perspectives on the parallel transport equation}
\label{sec:three-views-parallel-transport}

Let $\gamma \colon [0, T] \to M$ be a smooth curve on the manifold $M$ such that $\gamma(0) = p \in M$ and let $X \in \X(M)$ be a smooth vector field on $M$.
Depending on what we regard as the unknown in the parallel transport equation $\nabla_{\dot{\gamma}} X = 0$, we have three different types of equations (the third one being the least common and the one used in this paper).
Namely,
\begin{itemize}
\item $\nabla_{\dot{\gamma}} \dot{\gamma} = 0$, where $\gamma$ is the unknown.
  In local coordinates, it is straight-forward to see that this is a system of second order ordinary differential equations (ODEs) with initial condition $(\gamma(0), \dot{\gamma}(0)) = (p, v) \in T_{p} M$.
  The geodesics of $M$ are precisely the curves $\gamma$ that satisfy this equation.
\item $\nabla_{\dot{\gamma}} X = 0$, where $X$ is the unknown.
  This is a system of first order ODEs where the initial condition is $X(\gamma(0)) = X(p) \in T_p M$.
  The solution is a vector field $X(\gamma(s)) \in T_{\gamma(s)} M$ defined along the points of the known curve $\gamma(s)$.
  This vector field is the parallel transport of the initial vector $X(p)$ along $\gamma$.
\item $\nabla_{\dot{\gamma}} X = 0$, where $\gamma$ is the unknown.
  This yields a linear (algebraic) equation for $\dot{\gamma}$, which in turn allows us to pose an ODE for $\gamma$ and subsequently solve it.
  The solution is a curve joining the points of $M$ for which the known vector $X$ is parallel-transported along $\gamma$.
\end{itemize}


%% file: statements-and-declarations.tex
\section*{Declarations}

\noindent\textbf{Ethical approval:} Not applicable.

\noindent\textbf{Competing interests:} The authors declare no competing interests.

\noindent\textbf{Authors' contributions:} \emph{Juan M. Bello-Rivas:} Conceptualization, Methodology, Software, Validation, Formal analysis, Investigation, Writing -- Original Draft, Writing -- Review \& Editing, Visualization.
\emph{Anastasia Georgiou:} Validation, Writing -- Review \& Editing.
\emph{John Guckenheimer:} Conceptualization, Validation, Formal analysis, Writing -- Review \& Editing.
\emph{Ioannis G. Kevrekidis:} Conceptualization, Methodology, Validation, Formal analysis, Writing -- Review \& Editing, Supervision, Project administration, Funding acquisition.

\noindent\textbf{Funding:} This work was partially supported by DARPA, an AFOSR MURI, and the US Department of Energy.

\noindent\textbf{Availability of data and materials:} Proof of concept code for the algorithm presented in the paper is available at \url{https://github.com/jmbr/staying-the-course-code}.
